\documentclass{article}

\usepackage{lmodern}
\usepackage[T1]{fontenc}
\usepackage[left=1.5in,right=1.5in,bottom=1.25in]{geometry}

\usepackage{amsthm}
\usepackage{amsmath}
\usepackage{amsfonts}
\usepackage{authblk}
\usepackage{hyperref}
\usepackage[
  backend=biber,
  natbib=true,
  style=apa,
  sorting=nyt,
  maxcitenames=1,
  url=false, 
  doi=false,
  giveninits=true
]{biblatex}
\usepackage{graphicx}
\usepackage{centernot}
\usepackage{algorithm2e}
\usepackage{todo}
\usepackage{tikz}
\usepackage{wrapfig}
\usepackage{subcaption}
\usepackage{commath}
\usepackage{subcaption}
\usetikzlibrary{bayesnet}
\usetikzlibrary{arrows.meta, positioning, shapes, fit, backgrounds, calc}
\usepackage{float}

\hypersetup{
  colorlinks,
  linkcolor={red!50!black},
  citecolor={blue!50!black},
  urlcolor={blue!50!black}
}
\DeclareNameAlias{author}{last-first}
\newcommand\ci{\perp\!\!\!\perp}
\newcommand{\tto}{{\theta^{(1)}}}
\newcommand{\ttz}{{\theta^{(0)}}}
\newcommand{\ttt}{{\theta^{(t)}}}

\title{Causal posterior estimation}

\author[$*$]{Simon Dirmeier}
\author[$\dag\ddag$]{Antonietta Mira}
\affil[$*$]{Swiss Data Science Center, ETH Zurich, Switzerland}
\affil[$\dag$]{Università della Svizzera italiana, Switzerland}
\affil[$\ddag$]{Insubria University, Italy}
\date{}

\begin{document}
\maketitle

\begin{abstract}
We present Causal Posterior Estimation (CPE), a novel method for Bayesian inference in simulator models, where evaluating the likelihood function is intractable or computationally expensive, but generating outputs given parameter values is straightforward. CPE approximates the posterior distribution using flow matching while directly incorporating the conditional dependence structure induced by the model's graphical representation into the neural network architecture. Across extensive experiments, we demonstrate that hard-coding these conditional dependencies into the network, rather than requiring them to be learned from data, enables CPE to achieve highly accurate posterior inference that matches or outperforms state-of-the-art baselines.
\end{abstract}
\section{Introduction}
Consider the problem of inferring the Bayesian posterior distribution $\pi(\theta | x) \propto \pi(x | \theta) \pi(\theta)$ when the likelihood function $\pi(x | \theta)$ is intractable to evaluate, but simulating from the joint distribution $\pi(x, \theta)$ is feasible. \textit{Simulation-based inference} (SBI) methods address this problem by training neural networks on synthetic data sets generated from the forward model to approximate the posterior.

Recently, continuous neural density estimators based on flow matching and diffusion models \citep{song2021scorebased, lipman2023flow, liu2023flow} have become the state of the art in SBI. Both classes of methods transport a simple base distribution $q_0$, e.g., a standard Gaussian, to an approximation of the posterior by integrating an ordinary differential equation
\begin{equation}
    \frac{\mathrm{d}\theta^{(t)}}{\mathrm{d}t} = v_t\big(\theta^{(t)}, x\big), \quad \theta^{(0)} \sim q_0,
\end{equation}
Flow matching regresses $v_t$ directly onto a target vector field \citep{lipman2023flow, wildberger2023flow}, whereas score-based diffusion models learn the score $\nabla_\theta \log \pi_t(\theta | x)$ of the noised posterior and obtain $v_t$, up to a reversal of time, through the probability flow ODE \citep{song2021scorebased, sharrock2024sequential, gloeckler2024allinone},
\begin{equation}
    v_t(\theta, x) = f(\theta, t) - \tfrac{1}{2} g(t)^2 \, \nabla_\theta \log \pi_t(\theta | x),
\end{equation}
where $f$ and $g$ are the drift and diffusion coefficients of the noising process. 
In comparison to ABC algorithms \citep{sisson2007sequential,beaumont2009adaptive,sisson2018handbook,albert2015simulated,albert2025simulated}, neural SBI methods based on flow- and score-matching scale gracefully to high-dimensional spaces achieving high posterior accuracy on a wide range of tasks \citep{gloeckler2024allinone, sharrock2024sequential, wildberger2023flow}.

However, a critical limitation of current continuous SBI methods lies in their architectural design. Most approaches parameterize the vector field with dense architectures, such as multi-layer perceptrons (MLPs), which must learn the dependence structure between parameters and observations entirely from data. They thereby ignore the conditional independencies implied by the Bayesian graphical model, i.e., the prior and posterior programs \citep{van2018introduction}, which leads to large networks and reduced sample efficiency. Transformer-based approaches can encode this structure through attention masks \citep{gloeckler2024allinone}, but require a comparatively large number of parameters.

To address this, we introduce \textit{Causal Posterior Estimation} (CPE), a novel SBI method that restricts the hypothesis space of the neural density estimator by encoding the causal relationships of the Bayesian model directly into a continuous normalizing flow. Instead of relying on a network to discover the dependence structure from data, CPE parameterizes the flow's vector field with a causal block-autoregressive neural network whose dependencies structurally follow the posterior program (Figure~\ref{fig:figure_one}). Specifically, our contributions are as follows: 
\begin{itemize}
    \item \textbf{Causal block neural network architecture}: We propose a novel continuous normalizing flow whose vector field is parameterized by a structural masking mechanism that explicitly respects the graphical model of the posterior.
    \item \textbf{Prior-guided vector fields}: We parameterize the vector field as a convex combination of the current state and the network output, using the prior directly as the base distribution.
    \item \textbf{Accurate and efficient posterior inference}: Empirically, CPE achieves on average higher accuracy than state-of-the-art baselines across multiple benchmark tasks, while requiring considerably fewer trainable parameters.
\end{itemize}
 \definecolor{cpeone}{HTML}{C2590B}   
\definecolor{cpetwo}{HTML}{2B5D8C}   

\newcommand{\cpe}[2]{%
  \tikz[baseline=(cn.base)]{%
    \node[circle, fill=#1, text=white, inner sep=1.15pt,
          font=\scriptsize\bfseries] (cn) {#2};}}
\newcommand{\ub}[3]{\substack{\textsf{\scriptsize\color{#1}#2}\\[2pt]\cpe{#1}{#3}}}
\newcommand{\ob}[3]{\substack{\cpe{#1}{#3}\\[2pt]\textsf{\scriptsize\color{#1}#2}}}
\begin{figure}[t]
\centering
\begin{tikzpicture}[
  font=\small,
  cvar/.style  ={circle, draw=black!55, line width=.6pt, inner sep=0pt,
                 minimum size=15pt, font=\scriptsize},
  cobs/.style  ={cvar, fill=black!10},
  carr/.style  ={-{Stealth[length=4pt,width=3pt]}, draw=black!55, line width=.6pt},
  cbig/.style  ={-{Stealth[length=6pt,width=5pt]}, draw=black!25, line width=1.6pt},
  cflow/.style ={-{Stealth[length=4pt,width=3pt]}, draw=black!35, line width=.5pt},
  cstep/.style ={circle, fill=cpeone, text=white, font=\scriptsize\bfseries,
                 inner sep=1.7pt},
  cttl/.style  ={font=\sffamily\scriptsize\bfseries, text=black!75, anchor=center},
  ccap/.style  ={font=\sffamily\scriptsize, text=black!60, anchor=north},
]

\node[cttl] (t1) at (1.05,2.75) {prior program};
\node[cstep, anchor=east] at ($(t1.west)+(-0.10,0)$) {1};
\node[cvar] (a1) at (0.45,1.85) {$\theta_1$};
\node[cvar] (a2) at (0.45,1.15) {$\theta_2$};
\node[cvar] (a3) at (1.35,1.15) {$\theta_3$};
\node[cobs] (ax) at (0.90,0.45) {$x$};
\draw[carr] (a1) -- (a2); \draw[carr] (a2) -- (ax); \draw[carr] (a3) -- (ax);

\draw[carr] (2.25,1.15) -- node[ccap, anchor=south, pos=.5, yshift=1pt]
  {invert} (3.15,1.15);

\node[cttl] (t2) at (4.25,2.75) {posterior program};
\node[cstep, anchor=east] at ($(t2.west)+(-0.10,0)$) {2};
\node[cobs] (bx) at (4.10,1.85) {$x$};
\node[cvar] (b2) at (3.65,1.15) {$\theta_2$};
\node[cvar] (b3) at (4.55,1.15) {$\theta_3$};
\node[cvar] (b1) at (3.65,0.45) {$\theta_1$};
\draw[carr] (bx) -- (b2); \draw[carr] (bx) -- (b3); \draw[carr] (b2) -- (b1);
\draw[carr, dashed] (b3) -- (b2);
\node[ccap] at (4.10,0.12) {$\omega=(\theta_3,\theta_2,\theta_1)$};

\draw[cbig] (5.30,1.15) -- node[ccap, anchor=south, pos=.5, yshift=2pt]
  {CI statements} (6.60,1.15);

\node[cttl] (t3) at (9.05,2.75) {posterior sampling};
\node[cstep, anchor=east] at ($(t3.west)+(-0.10,0)$) {3};

\draw[fill=black!8, draw=black!45, line width=.6pt, smooth, samples=50,
      domain=-0.80:0.80]
  plot ({7.55-0.55*exp(-\x*\x*3.5)}, {1.15+\x}) -- cycle;
\draw[fill=cpetwo!18, draw=cpetwo!65, line width=.6pt, smooth, samples=50,
      domain=-0.80:0.80]
  plot ({10.15+0.38*exp(-(\x-0.22)*(\x-0.22)*12)
              +0.58*exp(-(\x+0.25)*(\x+0.25)*18)}, {1.15+\x}) -- cycle;

\foreach \s/\e in {0.62/0.30, -0.55/0.10}
  \draw[cflow] (7.55,{1.15+\s}) to[out=0, in=180] (10.15,{1.15+\e});
\draw[-{Stealth[length=4pt,width=3pt]}, draw=cpeone, line width=.8pt]
  (7.55,1.20) to[out=0, in=180] (10.15,1.10);
\foreach \i in {1,...,5}
  \fill[cpeone] ({7.55+\i*0.433},{1.20-0.1*\i/5}) circle (1.1pt);

\node[draw=cpeone!50, fill=cpeone!8, rounded corners=3pt, line width=.5pt,
      inner xsep=5pt, inner ysep=3pt] (vf) at (8.70,2.15) {$v_t(\theta_t,x)$};
\draw[carr, draw=cpeone!55] (vf.south) -- (8.70,1.58);
\node[cobs] (xo) at (10.45,2.15) {$x_{\mathrm{obs}}$};
\draw[carr] (xo) -- (vf);

\node[ccap] at (7.35,0.30) {$\pi(\theta)$};
\node[ccap] at (10.15,0.30) {$\hat\pi(\theta\mid x_{\mathrm{obs}})$};
\end{tikzpicture}
\caption{\textbf{Causal Posterior Estimation}:
\textbf{(1)}~The prior program, i.e., the generative model whose parameters $\theta$ we aim to infer.
\textbf{(2)}~The prior program is inverted into the posterior program. From the posterior program, conditional dependencies are computed and incorporated into a NN architecture.
\textbf{(3)}~Posterior inference is conducted by drawing a sample from the prior $\pi(\theta)$ and transporting it to the posterior $\pi(\theta|x_{\mathrm{obs}})$ by integrating the vector field $v_t(\theta^{(t)},x)$.}
\label{fig:figure_one}
\end{figure} 

\section{Preliminaries}
\label{sec:preliminaries}
In conventional Bayesian inference, a tractable likelihood enables posterior inference via Markov chain Monte Carlo or variational inference \citep{brooks2011handbook, blei2017variational}. Neural SBI methods avoid likelihood evaluations altogether and instead use a data set of $N$ simulated pairs $\{(x_n, \theta_n)\}_{n=1}^N \sim \pi(x, \theta)$ to fit neural approximations of the likelihood, the posterior, or the likelihood-to-evidence ratio. Below, we introduce the continuous generative models that underpin our approach.

\subsection{Continuous Normalizing Flows and Flow Matching}
Continuous normalizing flows (CNFs) \citep{chen2018neural} are neural density estimators based on pushforward measures. A neural network with weights $\phi$ defines a time-dependent vector field $v_t: [0, 1] \times \mathbb{R}^{d_\theta} \rightarrow \mathbb{R}^{d_\theta}$ (we omit the subscript $\phi$ for brevity), which induces the ODE
\begin{equation}
    \frac{\mathrm{d}}{\mathrm{d}t} \Lambda_t(\theta) = v_t\big(\Lambda_t(\theta)\big), \quad \Lambda_0(\theta) = \theta,
\label{eqn:ode-npe}
\end{equation}
where the flow $\Lambda_t$ is a time-dependent diffeomorphism. We write $\theta^{(t)} = \Lambda_t(\theta^{(0)})$ for the trajectory starting at $\theta^{(0)}$. Unlike discrete normalizing flows, CNFs impose no architectural constraints, such as explicit invertibility, on the network. Given a base distribution $q_0$ and a sample $\theta^{(0)} \sim q_0$, integrating Equation~\eqref{eqn:ode-npe} yields a sample $\theta^{(1)}$ from the pushforward distribution, whose density follows from the instantaneous change of variables formula
\begin{equation}
    q_1\big(\theta^{(1)}\big) = q_0\big(\theta^{(0)}\big) \exp\left( -\int_0^1 \mathrm{div}\, v_t\big(\theta^{(t)}\big) \, \mathrm{d}t \right).
\label{eqn:cnf}
\end{equation}

Training a CNF by maximizing the likelihood in Equation~\eqref{eqn:cnf} is computationally expensive, since every gradient step requires solving the ODE. \textit{Flow matching} instead regresses $v_t$ onto a known conditional target vector field $u_t$, yielding a simple simulation-free least-squares objective \citep{lipman2023flow}. The target vector field is commonly obtained by drawing $\theta^{(0)}$ from the base distribution and $\theta^{(1)}$ independently from the data distribution, and moving from $\theta^{(0)}$ to $\theta^{(1)}$ along a straight line at constant velocity $\theta^{(1)} - \theta^{(0)}$ \citep{liu2023flow}
\begin{equation}
    \theta^{(t)} = t \, \theta^{(1)} + (1 - t) \, \theta^{(0)}, \qquad u_t\big(\theta^{(t)} | \theta^{(0)}, \theta^{(1)}\big) = \theta^{(1)} - \theta^{(0)},
\label{eqn:fm-path}
\end{equation}
which, for a standard Gaussian base distribution, recovers the optimal transport path of \citet{lipman2023flow} with $\sigma_\text{min} = 0$. Flow matching is closely related to denoising score matching \citep{song2021scorebased}; the two differ primarily in the choice of interpolation and regression target, and both allow deterministic sampling with ODE solvers.

\subsection{Neural Posterior Estimation (NPE)}
Modern NPE methods \citep{wildberger2023flow, sharrock2024sequential, gloeckler2024allinone} approximate the posterior with a {conditional} vector field $v_t: [0, 1] \times \mathbb{R}^{d_\theta} \times \mathbb{R}^{d_x} \rightarrow \mathbb{R}^{d_\theta}$. For flow matching \citep{wildberger2023flow}, $v_t$ is trained by regressing it onto the target $u_t$ in Equation~\eqref{eqn:fm-path} while conditioning on simulated observations $x \sim \pi(x | \theta^{(1)})$:
\begin{equation}
    \hat{\phi} = \arg\min_\phi \mathbb{E}_{\substack{t \sim \mathcal{U}(0, 1),\ \theta^{(0)} \sim q_0 \\ \theta^{(1)} \sim \pi(\theta),\ x \sim \pi(x | \theta^{(1)})}} \left[ \big\| v_t\big(\theta^{(t)}, x\big) - \big(\theta^{(1)} - \theta^{(0)}\big) \big\|^2 \right].
\label{eqn:npe-loss}
\end{equation}
The resulting conditional density $q_1(\theta | x)$ serves as an \textit{amortized} posterior approximation, i.e., $q_1(\theta | x) \approx \pi(\theta | x)$. Given an arbitrary observation $x_\text{obs}$, a posterior sample is obtained by drawing $\theta^{(0)} \sim q_0$ and integrating the ODE with vector field $v_t(\cdot, x_\text{obs})$ from $t = 0$ to $t = 1$.

\section{Causal Posterior Estimation}
\label{sec:method}
We introduce Causal Posterior Estimation (CPE), an SBI method that approximates the posterior $\pi(\theta | x)$ with a CNF whose vector field $v_t$ embeds the conditional dependence structure of the posterior program. Since this structure is known from the model specification, encoding it restricts the hypothesis space of the estimator rather than leaving the network to learn it from data. At its core, the vector field is a composition of $L$ blocks (Figure~\ref{fig:method}),
\begin{equation}
    v_t\big(\theta^{(t)}, x\big) = \big(g^{(L)}_{t, x} \circ \cdots \circ g^{(1)}_{t, x}\big)\big(\theta^{(t)}\big),
\label{eqn:vector-field}
\end{equation}
where each block $g^{(l)}_{t, x}$ is conditioned on the time $t$ and the observation $x$. In the remainder of this section, we describe CPE in three parts: the prior-guided blocks $g^{(l)}$ (Section~\ref{sec:prior-base-distribution}), the causal block-autoregressive layers inside each block (Section~\ref{sec:cond-dep}), and the training and sampling procedure (Section~\ref{sec:training-sampling}). Architectural details can be found in Appendix~\ref{app:model-details}.

\begin{figure}[h]
\centering
\begin{tikzpicture}[
  font=\small,
  >={Stealth[length=4pt]},
  gbox/.style={draw=cpeone!55, fill=cpeone!8, rounded corners=3pt,
               minimum width=1.05cm, minimum height=0.62cm, line width=.6pt},
  fbox/.style={draw=cpetwo!75, fill=cpetwo!18, rounded corners=2pt,
               minimum width=0.78cm, minimum height=0.55cm, line width=.6pt},
  arr/.style={->, line width=.6pt, black!65},
  num/.style={circle, fill=#1, text=white, font=\scriptsize\bfseries\sffamily,
              inner sep=0pt, minimum size=11pt},
]
\coordinate (pnw) at (-0.25,-0.8);
\coordinate (pse) at (7.55,-3.35);
\begin{scope}[on background layer]
  \draw[fill=cpeone!6, draw=cpeone!55, rounded corners=5pt, line width=.6pt]
    (pnw) rectangle (pse);
\end{scope}
\node[num=cpeone] (n1) at ($(pnw)+(0.27,-0.27)$) {1};
\node[anchor=west, text=cpeone!90!black, font=\small] at ($(n1.east)+(2pt,0)$)
  {$g^{(l)}_{t, x}(u) = \alpha_l\, u + (1-\alpha_l)\, \lambda^{(l)}_t(u, x)$};

\node (u) at (0.1,-2.1) {$u$};
\coordinate (split) at (0.55,-2.1);
\node[fbox] (f1) at (1.5,-2.1) {$f^{(1)}$};
\node       (fd) at (2.45,-2.1) {$\cdots$};
\node[fbox] (fk) at (3.4,-2.1) {$f^{(k)}$};
\node       (fd2) at (4.35,-2.1) {$\cdots$};
\node[fbox] (fK) at (5.3,-2.1) {$f^{(K)}$};
\node[draw, circle, inner sep=0pt, minimum size=9pt, line width=.6pt, black!70] (plus) at (6.95,-2.1) {$+$};

\node[fit=(f1)(fK), inner xsep=5pt, inner ysep=5pt] (lam) {};
\begin{scope}[on background layer]
  \draw[fill=cpetwo!7, draw=cpetwo!60, rounded corners=4pt, line width=.6pt]
    (lam.south west) rectangle (lam.north east);
\end{scope}

\draw[line width=.6pt, black!65] (u) -- (split);
\draw[arr] (split) -- (lam.west |- split);
\draw[arr] (f1) -- (fd);
\draw[arr] (fd) -- (fk);
\draw[arr] (fk) -- (fd2);
\draw[arr] (fd2) -- (fK);
\draw[arr] (lam.east |- split) -- node[below, font=\scriptsize] {$1-\alpha_l$} (plus);
\draw[arr] (split) |- (6.95,-1.45) -- (plus);
\node[anchor=west, font=\scriptsize, inner sep=1pt] at (6.97,-1.72) {$\alpha_l$};

\node[gbox] (gprev) at (-1.35,-2.1) {$g^{(l-1)}_{t, x}$};
\node       (dprev) at (-2.45,-2.1) {$\cdots$};
\node[gbox] (gnext) at (8.65,-2.1)  {$g^{(l+1)}_{t, x}$};
\node       (dnext) at (9.75,-2.1)  {$\cdots$};
\draw[arr] (dprev) -- (gprev);
\draw[arr] (gprev) -- (u);
\draw[arr] (plus) -- (gnext);
\draw[arr] (gnext) -- (dnext);

\node[text=black!70, font=\scriptsize] (cond) at (1.975,-3.05) {$(t, x)$};
\draw[arr] (cond) -- (cond |- lam.south);

\node[anchor=south east, text=cpetwo, font=\small] (lab2) at ($(pse)+(-0.1,0.08)$)
  {\textsf{$\lambda^{(l)}_t$ with $K$ causal block layers}};
\node[num=cpetwo, anchor=east] at ($(lab2.west)+(-2pt,0)$) {2};

\end{tikzpicture}
\caption{Overview of CPE. The vector field
$v_t(\theta^{(t)}, x)$ is a composition of $L$ prior-guided blocks (Section~\ref{sec:prior-base-distribution}). Each block forms a trainable convex combination of its input $u$, with $u = \theta^{(t)}$ for the first block, and a causal block-autoregressive network $\lambda^{(l)}_t$ with $K$ causal block layers (Section~\ref{sec:cond-dep}).}
\label{fig:method}
\end{figure}

\subsection{Prior-Guided Vector Fields}
\label{sec:prior-base-distribution}
Inspired by the fact that, for conjugate exponential family models, the posterior mean is a convex combination of the prior mean and a data-dependent estimate \citep{diaconis1979conjugate}, we use the prior $\pi(\theta)$ directly as the base distribution, $q_0(\theta) = \pi(\theta)$, instead of starting the flow, for instance, from a standard Gaussian. We then form each block as a convex combination of its input and the output of a neural network:
\begin{equation}
    g^{(l)}_{t, x}(u) = \alpha_l \, u + (1 - \alpha_l) \, \lambda^{(l)}_t(u, x),
\label{eqn:convex-combination}
\end{equation}
where $u$ is the output of the previous block, with $u = \theta^{(t)}$ for the first block, $\alpha_l \in (0, 1)$ are trainable scalars, and each $\lambda^{(l)}_t$ is a time-dependent neural network with $K$ layers (Section~\ref{sec:cond-dep}). The convex combinations act as learnable skip connections between the state of the flow and the network outputs, which, together with the prior as base distribution, provide an inductive bias that can simplify optimization and improve the accuracy of the estimator.

\subsection{Causal Block-Autoregressive Layers}
\label{sec:cond-dep}
Consider a hierarchical model with four variables that forms a $v$-structure
\begin{equation}
    \pi(\theta_1, \theta_2, \theta_3, x) = \pi(\theta_1) \, \pi(\theta_2 | \theta_1) \, \pi(\theta_3) \, \pi(x | \theta_2, \theta_3),
\label{eqn:hierarchical-model}
\end{equation}
where the objective is to infer the joint posterior distribution $\pi(\theta_1, \theta_2, \theta_3 | x) \propto \pi(\theta_1, \theta_2, \theta_3, x)$ (Figure~\ref{fig:hierarchical-model}a). Dense architectures leave the dependence structure between parameters and data unrestricted and must learn it from data. Here, we aim to instead restrict the vector field to the conditional dependencies of the inverted graphical model, the \textit{posterior program} (Figure~\ref{fig:hierarchical-model}b).\footnote{Several approaches exist for inverting graphical models \citep{stuhlmueller2013learning, webb2018faithful}. We adopt the formulation of \citet{ambrogioni21automatic} for its computational simplicity. The algorithm of \citet{webb2018faithful} yields faithful inverses but denser graphs, and did not measurably improve performance in our preliminary experiments.} In this example, conditioning on the collider $x$ renders $\theta_2$ and $\theta_3$ dependent, so the posterior program contains the edge $\theta_3 \rightarrow \theta_2$ in addition to the reversed edges $x \rightarrow \theta_2$, $x \rightarrow \theta_3$ and $\theta_2 \rightarrow \theta_1$. At the same time, the posterior satisfies $\theta_1 \ci \theta_3 | \theta_2, x$, which we aim to encode by letting the vector field of $\theta_1$ depend on $\theta_2$ but not on $\theta_3$.

\begin{wrapfigure}{r}{0.5\textwidth}
\centering
\begin{subfigure}[b]{0.24\textwidth}
\centering
\begin{tikzpicture}
\node[const]                               (y) {$x$};
\node[const, above=0.75cm of y, xshift=-.75cm] (t2) {$\theta_2$};
\node[const, above=0.75cm of y, xshift=.75cm]  (t3) {$\theta_3$};
\node[const, above=0.75cm of t2]            (t1) {$\theta_1$};
\edge [style={-latex}, shorten <=2pt, shorten >=2pt]{t2,t3} {y};
\edge [style={-latex}, shorten <=2pt, shorten >=2pt] {t1} {t2};
\end{tikzpicture}
\caption{Prior program.}
\label{fig:prior-program}
\end{subfigure}
~
\begin{subfigure}[b]{0.24\textwidth}
\centering
\begin{tikzpicture}
\node[const]                               (y) {$x$};
\node[const, above=0.75cm of y, xshift=-.75cm] (t2) {$\theta_2$};
\node[const, above=0.75cm of y, xshift=.75cm]  (t3) {$\theta_3$};
\node[const, above=0.75cm of t2]            (t1) {$\theta_1$};
\edge [style={-latex}, shorten <=2pt, shorten >=2pt] {y} {t2,t3};
\edge [style={-latex}, shorten <=2pt, shorten >=2pt] {t2} {t1};
\edge [dashed, style={-latex}, shorten <=2pt, shorten >=2pt] {t3} {t2};
\end{tikzpicture}
\caption{Posterior program.}
\label{fig:posterior-program}
\end{subfigure}
\caption{Graphical models illustrating prior and posterior dependencies.}
\label{fig:hierarchical-model}
\end{wrapfigure}

\paragraph{Posterior Factorization}
To enforce the conditional dependence structure of the posterior program, we carefully parameterize the mapping $\lambda_t$. Assuming the posterior program forms a directed acyclic graph (DAG),\footnote{Generative models with temporal loops, such as dynamic Bayesian networks, can be unrolled into finite acyclic structures for inference.} we define a topological ordering over its parameter nodes, denoted by $\omega = \{ \omega_1, \dots, \omega_{d_\theta} \}$. The mapping $\lambda_t$ is constructed to respect the factorization:
\begin{equation}
    \pi(\theta | x) = \prod_{i=1}^{d_\theta} \pi\big(\theta_{\omega_i} | \mathrm{pa}(\theta_{\omega_i}), x\big),
    \label{eqn:factorization}
\end{equation}
where $\mathrm{pa}(\theta_{\omega_i})$ denotes the parents of $\theta_{\omega_i}$ in the posterior program.

\paragraph{Structural Masking}
We implement Equation~\eqref{eqn:factorization} by applying a lower-triangular structural mask $W \in \mathbb{R}^{d_\theta \times d_\theta}$ over features ordered by $\omega$. To avoid modeling spurious correlations, we strictly set $W_{ij} = 0$ wherever $\theta_i \ci \theta_j | \theta_{\setminus \{i,j\}}, x$ holds in the posterior program. For the posterior program in Figure~\ref{fig:posterior-program} with topological ordering $\omega = \{ 3, 2, 1 \}$, the structural matrix takes the form
\begin{equation}
W = \begin{bmatrix}
    g(W_{1,1}) & 0 & 0 \\
    W_{2,1} & g(W_{2,2}) & 0 \\
    0 & W_{3,2} & g(W_{3,3})
    \end{bmatrix},
\end{equation}
where $g: \mathbb{R} \rightarrow \mathbb{R}$ is an element-wise activation function\footnote{Using $g$ is not strictly necessary: we introduce it to render the model invertible for the discrete NF case in Appendix~\ref{app:model-details}.}.

\paragraph{Block-Matrix Augmentation}
To increase expressivity, we expand the masking strategy into a causal block neural autoregressive flow (cf. \citet{huang2018neural, cao2020block, ambrogioni21automatic}) by replacing scalar entries $W_{ij}$ with block matrices $B_{ij} \in \mathbb{R}^{d_{\mathrm{out}} \times d_{\mathrm{in}}}$ (see Figure~\ref{fig:method-mask}). The mapping $\lambda_t$ is then composed as a sequence of transformations $\lambda_t = f^{(K)} \circ \dots \circ f^{(1)}$, where each layer computes
\begin{equation}
    f^{(k)}(\theta) = \mathrm{act}\left(B^{(k)}\theta + b^{(k)}\right),
\label{eqn:linear-projection}
\end{equation}
with bias vector $b^{(k)}$ and activation function $\mathrm{act}(\cdot)$. The block sizes are configured with $d_{\mathrm{in}}^{(1)} = d_{\mathrm{out}}^{(K)} = 1$.

\begin{wrapfigure}{r}{0.34\textwidth}
\centering
\begin{tikzpicture}[font=\small]
\def\bs{0.85}   
\def\cells{3}   

\foreach \r/\rlab in {0/3, 1/2, 2/1}{
  \node[anchor=east] at (-0.08, {-(\r+0.5)*\bs}) {$\theta_{\rlab}$};
  \node[anchor=south] at ({(\r+0.5)*\bs}, 0.05) {$\theta_{\rlab}$};
}
\foreach \r/\c/\type in {0/0/d, 0/1/z, 0/2/z,
                         1/0/e, 1/1/d, 1/2/z,
                         2/0/z, 2/1/e, 2/2/d}{
  \ifthenelse{\equal{\type}{d}}{\def\fc{cpetwo!55}}{
    \ifthenelse{\equal{\type}{e}}{\def\fc{cpeone!45}}{\def\fc{white}}}
  \fill[\fc] ({\c*\bs}, {-\r*\bs}) rectangle ({(\c+1)*\bs}, {-(\r+1)*\bs});
  \ifthenelse{\equal{\type}{z}}{
    \node[text=black!35] at ({(\c+0.5)*\bs}, {-(\r+0.5)*\bs}) {$0$};}{
    \foreach \i in {1,...,2}{
      \draw[white, line width=.3pt]
        ({\c*\bs + \i*\bs/\cells}, {-\r*\bs}) -- ++(0, -\bs);
      \draw[white, line width=.3pt]
        ({\c*\bs}, {-\r*\bs - \i*\bs/\cells}) -- ++(\bs, 0);}}
}
\draw[black!40, line width=.5pt] (0,0) rectangle ({3*\bs}, {-3*\bs});
\foreach \i in {1,2}{
  \draw[black!40, line width=.5pt] ({\i*\bs}, 0) -- ++(0, {-3*\bs});
  \draw[black!40, line width=.5pt] (0, {-\i*\bs}) -- ++({3*\bs}, 0);}

\begin{scope}[shift={(0,{-3*\bs-0.25})}, font=\scriptsize]
  \fill[cpetwo!55] (0,0) rectangle ++(0.22,-0.22);
  \node[anchor=west] at (0.28,-0.11) {diagonal, $g(\cdot)$};
  \fill[cpeone!45] (0,-0.36) rectangle ++(0.22,-0.22);
  \node[anchor=west] at (0.28,-0.47) {edge in posterior program};
  \draw[black!40] (0,-0.72) rectangle ++(0.22,-0.22);
  \node[anchor=west] at (0.28,-0.83) {conditional independence};
\end{scope}
\end{tikzpicture}
\caption{Masked weight matrix $B^{(k)}$ for the posterior program in
Figure~\ref{fig:posterior-program}.}
\label{fig:method-mask}
\end{wrapfigure}
\paragraph{Conditioning Mechanism}
The vector field $v_t$ is conditioned on the time $t$ and the observation $x$. We embed $t$ with sinusoidal features, project the time embedding and the data $x$ through separate MLPs, and concatenate the resulting embeddings. A linear projection of this vector is added to the hidden units of all nodes after the first layer $f^{(1)}$ (cf. the conditioning in \citet{germain2015made,papamakarios2017masked}).

\subsection{Training and Sampling}
\label{sec:training-sampling}
We train CPE with the simulation-free rectified flow objective \citep{liu2023flow}, i.e., Equation~\eqref{eqn:npe-loss} with the prior as base distribution, $q_0 = \pi(\theta)$, and the vector field in Equation~\eqref{eqn:vector-field}. Training therefore requires no ODE integration, but merely a synthetic data set of pairs $(\theta^{(1)}, x)$ drawn from the prior and the simulator.

To obtain posterior samples given an observation $x_\text{obs}$, we draw $\theta^{(0)} \sim \pi(\theta)$ from the prior and integrate the ODE with vector field $v_t(\cdot, x_\text{obs})$ from $t = 0$ to $t = 1$. We consider two solvers: an adaptive Runge--Kutta 5(4) solver (CPE-RK) and an Euler solver with $20$ fixed steps (CPE-Euler).Because the vector field is trained on convex interpolations between prior and posterior samples (Equation~\eqref{eqn:fm-path}) we expect its trajectories to be close to straight, such that a coarse Euler discretization suffices. Also, since the ODE is not constrained to the support of the prior, individual samples can fall outside it and thus have zero prior density. As in rejection sampling, we discard such samples and draw new ones until the desired number of posterior samples is reached.

\section{Related work}
\label{sec:related}
\paragraph{Flow and denoising score matching} Flow and denoising score matching have been applied with great success to generative modeling \citep{ho2020diffusion,song2021scorebased,karras2022elucidate,lipman2023flow,liu2023flow,albergo2023building}, Bayesian and variational inference \citep{zhang2022path,vargas2023denoising,berner2024an,vargas2024transport}, and inverse problems \citep{kawar2022denoising,chung2023diffusion,dou2024diffusion}.

\paragraph{Simulation-based inference}
Early neural SBI methods approximate the posterior with mixture density networks or discrete normalizing flows \citep{papamakarios2016fast,lueckmann2017flexible,greenberg2019automatic,deistler2022truncated,radev2023jana}. More recent methods use continuous-time models trained with flow matching or score matching objectives \citep{geffner2023compositional,wildberger2023flow,sharrock2024sequential,gloeckler2024allinone,schmitt2024consistency,tong2024improving,gao2025diffusion}, which have since been extended to, e.g., hierarchical models \citep{charles2026tokenised}, mixed discrete-continuous or constrained parameter spaces \citep{boelts2026mixed,carrasco2026pawsterior}, and tall data \citep{linhart2024diffusion,touron2026theoretical}. We refer to \citet{deistler2025simulation,dax2026introduction,arruda2025diffusion} for recent overviews.

\paragraph{Structured inference architectures}
Encoding the conditional dependence structure of a probabilistic model into the inference network is not a novel approach. \citet{webb2018faithful} introduced faithful inversion of generative models for amortized inference, which \citet{weilbach2020structured,weilbach2023graphically} used to sparsify continuous normalizing flows and diffusion models. Masked weight matrices, which restrict each output of a network to a subset of its inputs, form the core of masked autoregressive flows \citep{germain2015made,papamakarios2017masked} and have been extended to encode graphical structure \citep{wehenkel2021graphical,silvestri2022embedded,chen2023structured}. Most related to our work, \citet{gloeckler2024allinone} use attention masks in a transformer-based score network to encode the dependence structure of the simulator. CPE differs from these approaches in three respects: it encodes the posterior program directly in the block-masked weight matrices of a comparatively small network rather than in the attention masks of a transformer, it combines this architecture with the prior as base distribution, and its layers are invertible, such that CPE can also be used as a discrete normalizing flow, whereas the networks of \citet{gloeckler2024allinone,weilbach2023graphically} are not.

\paragraph{Further related work}
Our use of the term \emph{causal} refers to the conditional independence structure of the prior and posterior programs and should not be confused with causal normalizing flows, which reflect a structural causal model and enable interventional and counterfactual queries \citep{javaloy2023causal,visentin2025robust}. Robust SBI methods, which account for model misspecification \citep{frazier2020model,ward2022robust,huang2023learning,kelly2024misspecificationrobust}, are largely agnostic to the architecture of the inference network and thus complementary to CPE. We discuss both lines of work in more detail in Appendix~\ref{app:related-work}.

\section{Experiments}
\label{sec:experiments}

We evaluate CPE extensively on several SBI benchmark models, the susceptible--infected--recovered (SIR) epidemiological model, and a realistic deep hierarchical model which is in similar forms often found in computational natural sciences.

We compare CPE against three state-of-the-art baselines: Flow Matching Posterior Estimation \citep[FMPE;][]{wildberger2023flow}, Neural Posterior Score Estimation \citep[PSE;][]{sharrock2024sequential}, and All-in-One Posterior Estimation \citep[AIO;][]{gloeckler2024allinone}. For AIO, we report results both with an undirected graph mask (AIO) and without a mask (AIO (no mask)) to be able to separate the effect of the mask from that of the transformer architecture. For FMPE, PSE, and AIO, we use the default configurations recommended in their original publications. To isolate the contribution of the causal masking mechanism (Section~\ref{sec:cond-dep}), we additionally compare against an ablation of CPE that uses a block-autoregressive network without causal masks, termed BNPE. All methods draw samples from the posterior using an adaptive Runge--Kutta 5(4) solver. For CPE, we additionally report results with an Euler solver with $T = 20$ steps.

For each benchmark task, we draw five pairs $(x_{\mathrm{obs}}, \theta) \sim \pi(x, \theta)$ from the generative model and approximate the posterior $\pi(\theta | x_{\mathrm{obs}})$ for each observation. Since all benchmark tasks have a tractable likelihood, which is not the case in real-world applications, we can obtain "ground-truth" posteriors with Markov chain Monte Carlo. We compare the approximate posteriors to these ground-truth posteriors for different simulation budgets $N$ using the H-min divergence \citep{zhao2022comparing, dirmeier2023simulation} and classifier two-sample tests \citep[C2ST;][]{lopezpaz2017revisiting}.64

See Appendices~\ref{app:experimental-models} and \ref{app:experimental-details} for details on the benchmark models and experimental setup, including computation of reference posterior distributions, training, and exact neural network architectures.

\subsection{SBI benchmarks}
We first evaluate CPE in an experimental setting using eight common SBI benchmark models (see Appendix~\ref{app:experimental-models} for details). 

CPE outperforms or matches all baselines on all eight benchmark models under both H-min and C2ST (Figure~\ref{fig:benchmark_tasks-hmin}; C2ST results in Appendix~\ref{app:additional-results}). Compared to BNPE, the non-causal variant of our architecture, CPE performs better on the majority of benchmark models and on par on the rest, indicating that the causal masking is critical to its accuracy.

CPE-Euler, which uses $20$ Euler steps, performs on par with CPE-RK, which uses an adaptive Runge--Kutta solver, and in some cases even outperforms it. This is consistent with the learned trajectories having low curvature (Section~\ref{sec:training-sampling}), such that posterior samples can be drawn with few solver steps without loss in accuracy.

In addition, CPE reaches this performance with generally fewer trainable parameters than the baselines. On the mixture model with distractors, CPE uses 47k parameters, against 172k, 191k and 191k for AIO, FMPE and NPSE, respectively. On the hyperboloid, two moons and Gaussian mixture models, CPE uses 71k, against 172k, 190k and 190k for the same baselines.
\begin{figure}
    \centering
    \includegraphics[width=1\textwidth]{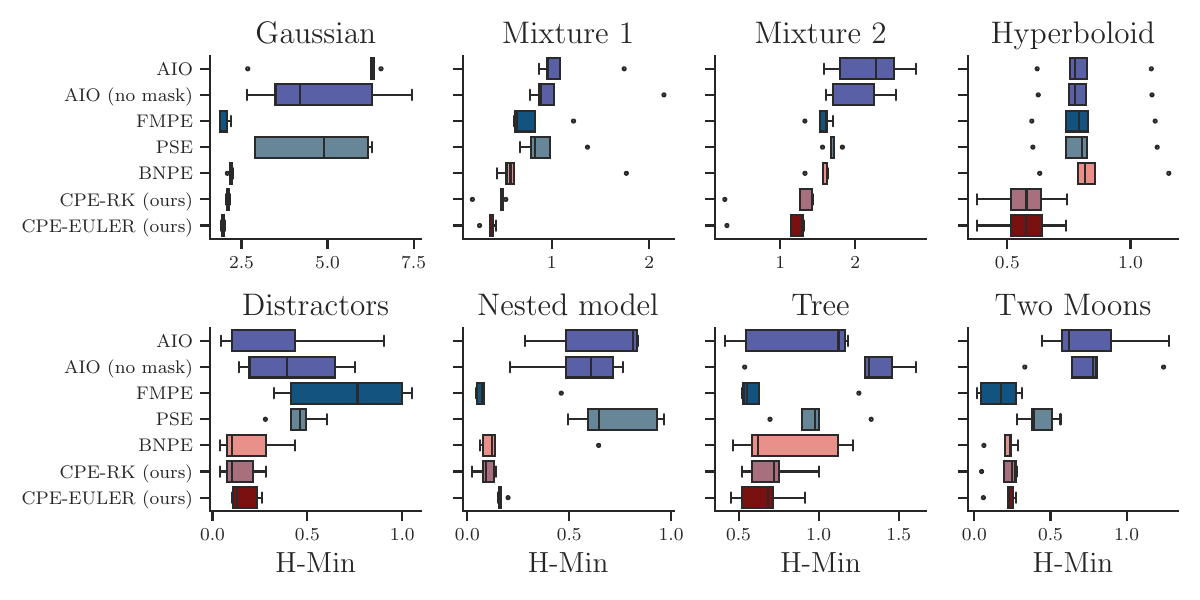}
    \caption{Performance of CPE and all baselines on eight common SBI benchmark models, measured by the H-min metric (lower is better).}
    \label{fig:benchmark_tasks-hmin}
\end{figure}
\begin{wrapfigure}{r}{0.5\textwidth}
\centering
\includegraphics[width=0.5\textwidth]{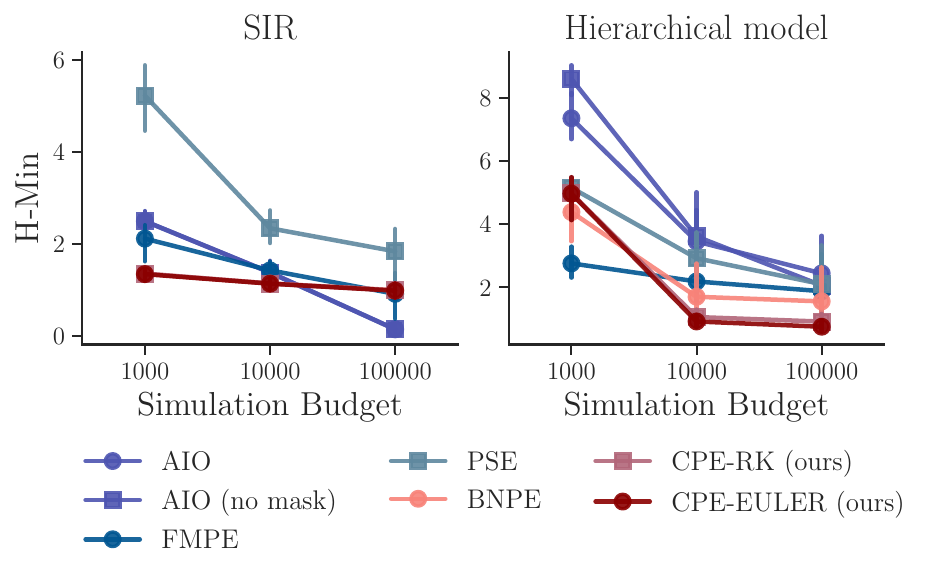}
\caption{Evaluations for the hierarchical and SIR models, measured by the H-min metric (lower is better).}
\label{fig:sir+hierarchical_model-evaluations}
\end{wrapfigure}

\subsection{SIR model}
We next evaluate CPE on a compartmental model from epidemiology. The model describes how the numbers of susceptible ($S$), infected ($I$) and recovered ($R$) individuals change as an infectious disease spreads through a population. It has only two parameters, the contact rate $\theta_0$ and the recovery rate $\theta_1$, but integrating its ODE produces a time series whose length is determined only by the time grid. Following \citet{lueckmann2021benchmarking}, we integrate the ODE for a population of $10^6$ individuals at $160$ equally spaced time points in $[0, 160]$. At each time point $t$, we observe $Y_t \sim \mathrm{Binomial}(1\,000, I_t / 10^6)$. Instead of conditioning on the full series, we reduce it to a five-dimensional vector of summary statistics
\begin{equation*}
s(Y) = \left[ \max_t Y_t,\; \operatorname*{arg\,max}_t Y_t,\; \sum_t Y_t,\; \bar{Y},\; \widehat{\mathrm{Var}}(Y) \right],
\end{equation*}
where $\bar{Y}$ and $\widehat{\mathrm{Var}}(Y)$ denote the empirical mean and variance of the observed series.%
\footnote{CPE can also be applied to the raw time series directly. This, however, requires an appropriate summary network in addition to the architecture described in Appendix~\ref{app:cpe-architecture}.}

Here, both CPE variants outperform all baselines, except at the largest simulation budget of $N = 100\,000$, where AIO performs best, consistent with the results reported by \citet{gloeckler2024allinone} (Figure~\ref{fig:sir+hierarchical_model-evaluations}, left panel). At this budget, the posterior predictive distributions of AIO also reproduce the observed data closely (Appendix~\ref{app:sir-ppds}). When the summary statistics are not sufficient for the parameters (but merely informative, as in our case), posterior predictive distributions of models trained on them typically show considerably higher variance (cf. posterior and posterior predictive distributions in Appendix~\ref{app:additional_results_sir}). This suggests that the transformer-based network of AIO is (as expected) more expressive than the MLP-based networks of FMPE, PSE, and CPE, and hence better suited to models using non-sufficient summary statistics. Conversely, CPE already yields convincing results at small simulation budgets, whereas AIO and FMPE require a large number of simulations.

\subsection{Hierarchical model}
\label{sec:hierarchical_model}
To show that CPE is particularly effective when the simulator has a hierarchical or nested structure, we evaluate it on the following generative model:
\begin{equation*}
\begin{aligned}
\sigma &\sim \text{HalfNormal}(1) \\
\gamma &\sim \mathcal{N}_2(0, I) \\
\beta_k \mid \gamma &\sim \mathcal{N}_2(\gamma, I), \quad k = 1, 2, 3 \\
x \mid \beta, \sigma &\sim \mathcal{N}\big([\beta_1, \beta_2, \beta_3]^\top, \sigma^2 I\big)
\end{aligned}
\end{equation*}
The model has nine parameters, including a positive scale parameter $\sigma$ that governs the observation noise. Hierarchical models of this kind can be difficult to infer even for gradient-based MCMC samplers, since dependencies between levels of the hierarchy, and between location and scale parameters induce degenerate posterior geometries. At the same time, such models are ubiquitous in the natural sciences, for instance in computational biology and neuroscience, which underlines their practical relevance.

We evaluate CPE with simulation budgets of $N \in \{1\,000, 10\,000, 100\,000\}$. At the smallest budget of $N = 1\,000$, FMPE and BNPE, the non-causal variant of our method, outperform CPE (Figure~\ref{fig:sir+hierarchical_model-evaluations}, right panel). For $N \geq 10\,000$, CPE outperforms all baselines. Qualitatively, the posterior inferred by CPE matches the reference MCMC posterior considerably better than that of AIO, which, like CPE, accounts for the statistical dependencies induced by the graphical model (Figure~\ref{fig:hierarchical_model-some_posteriors}). While both methods identify the top-level parameters $\gamma$ well, AIO overestimates the scale $\sigma$, which in turn widens its posteriors over the lower-level parameters $\beta$ (see Appendix~\ref{app:additional_results_hierarchical_model} for posterior and posterior predictive distributions of all methods).

\begin{figure}[t]
\centering
\includegraphics[width=1\textwidth]{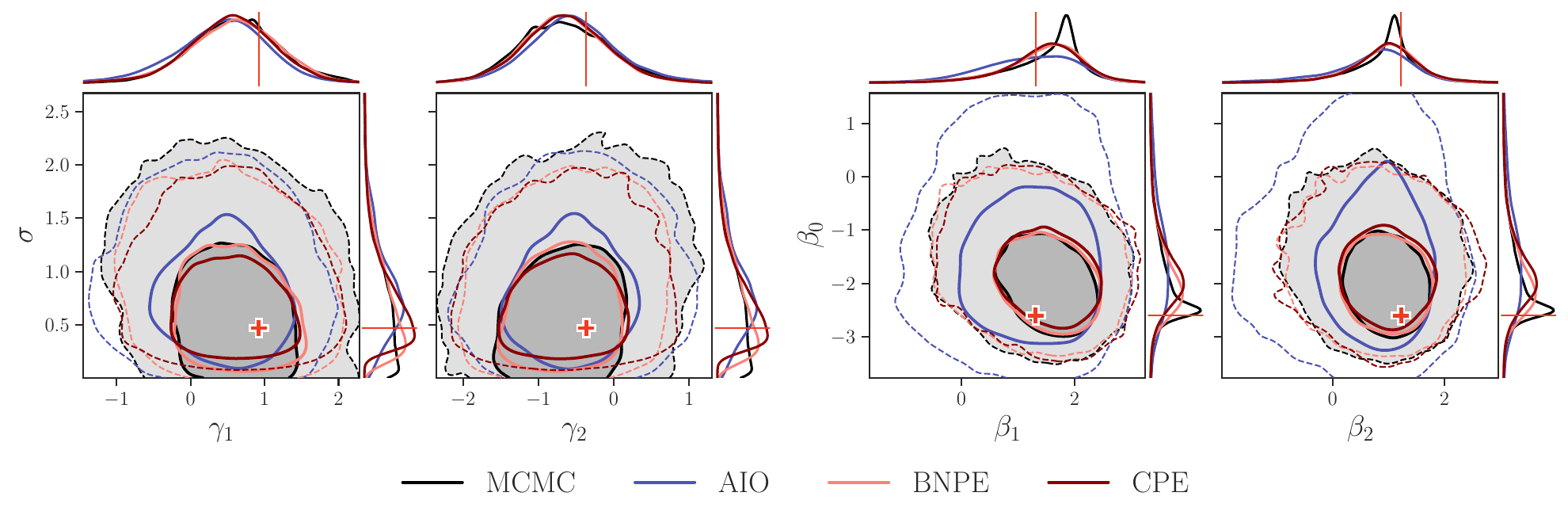}
\caption{Posterior approximations of the compared methods for the hierarchical model (Section~\ref{sec:hierarchical_model}). Contours enclose 68\% and 95\% of the posterior mass.}
\label{fig:hierarchical_model-some_posteriors}    
\end{figure}

\section{Limitations}
Our method has several limitations. In particular, the block matrix structure scales poorly with the number of parameters. While this is negligible for models with a small number of parameters (e.g., on the order of $10$), it can become inefficient when hundreds of parameters need to be inferred. Transformer-based methods avoid this issue by embedding the data and parameters in a vector space of fixed dimension \citep{gloeckler2024allinone,patacchiola2024transformer}. Future research could investigate how the expressivity of our approach can be combined with such efficient representations. 

Furthermore, we invert the generative model by naively reversing its edges. While this was sufficient for the models in Section~\ref{sec:experiments}, it may fail to capture all conditional dependencies correctly \citep{stuhlmueller2013learning,webb2018faithful}. In future work, we aim to account for these dependencies explicitly to further improve the accuracy of our method. 

Finally, the block matrices in our approach are dense within the sparsity pattern induced by the graphical model, so their memory and computational costs grow quadratically with the number of parameters. Several structured parameterizations could reduce this cost. For dense or weakly structured posterior programs, the blocks could be restricted to low-rank factorizations $B_{ij} = \xi_{ij} \zeta_{ij}^\top$ with trainable vectors $\xi_{ij}$ and $\zeta_{ij}$. For sparse posterior programs, the block matrices resemble structured semiseparable or block-banded matrices \citep{fu2023monarch, dao2024transformers}. \citet{dao2024transformers} show that semiseparable matrices, i.e., matrices whose submatrices below the diagonal have low rank, admit matrix--vector products in time linear in the sequence length. They use this structure to relate state space models to masked attention. Enforcing such structure would allow CPE to use these algorithms and could thereby address the scaling limitation discussed above. We have neither enforced such structures nor made use of these algorithms in this work. Investigating how structured low-rank matrices can be incorporated into SBI methods, and how much accuracy they sacrifice in exchange for computational efficiency, is an interesting direction for future research.

\section{Conclusion}
\label{sec:conclusion}
We presented a simple but powerful approach to SBI, called \textit{Causal Posterior Estimation} (CPE), which incorporates the conditional dependence structure of the prior and posterior programs into a continuous normalizing flow. CPE parameterizes the flow's vector field with block-masked layers that follow the inverted graphical model, uses the prior as base distribution, and formulates the vector field as a convex combination of the current state and the network output. Across several benchmark tasks, CPE achieves state-of-the-art performance while generally requiring considerably fewer trainable parameters than dense and transformer-based baselines such as FMPE, PSE and AIO. The benefit of encoding the model structure is most apparent for hierarchical models, where the posterior inferred by CPE closely matches the reference MCMC posterior.

While we focus on a continuous normalizing flow in this work, CPE can equally be implemented as a discrete normalizing flow (see Appendix~\ref{app:model-details}, where we derive the log-determinants of the network for the discrete case). More generally, the idea of restricting the hypothesis space of a neural density estimator to the conditional dependencies of the underlying model is not tied to a particular generative framework and could be transferred to score-based or transformer-based SBI methods. Promising directions for future work include faithful inversions of the graphical model and structured low-rank parameterizations of the block matrices for models with many parameters. 

\section*{Acknowledgements}
This research was supported by the Swiss National Science Foundation (Grant No. $200021\_208249$).

\newpage
\printbibliography


\appendix
\section{Model details}
\label{app:model-details}

\newtheorem{theorem}{Theorem}
\newtheorem{proposition}[theorem]{Proposition}
\newtheorem{lemma}[theorem]{Lemma}

CPE is a normalizing-flow-based neural posterior estimation method. In this section, we derive the continuous normalizing flow (CNF) variant of CPE introduced in the main text from first principles (Appendix~\ref{app:cpe-cnf}). Since the block transforms of CPE can be made invertible, we additionally describe a variant based on discrete normalizing flows (Appendix~\ref{app:cpe-nf}). Throughout, $q_0$ denotes the base distribution on $\mathbb{R}^{d_\theta}$, $\nabla_\theta v$ the Jacobian of a vector field $v: \mathbb{R}^{d_\theta} \to \mathbb{R}^{d_\theta}$, and
\[
\mathrm{div}\, v(\theta) = \sum_{i=1}^{d_\theta} \frac{\partial v_i(\theta)}{\partial \theta_i} = \operatorname{tr}\big(\nabla_\theta v(\theta)\big)
\]
its divergence. As in the main text, we omit the dependence of the vector field on the network weights $\phi$.

\subsection{CPE with continuous normalizing flows}
\label{app:cpe-cnf}

\subsubsection{Continuous normalizing flows}
\label{app:cnf}
Let $v: [0, 1] \times \mathbb{R}^{d_\theta} \to \mathbb{R}^{d_\theta}$, $(t, \theta) \mapsto v_t(\theta)$, be a time-dependent vector field that is continuous in $t$ and continuously differentiable in $\theta$ with bounded Jacobian. By the Picard--Lindel\"of theorem, the ODE
\begin{equation}
\frac{\mathrm{d}}{\mathrm{d}t} \Lambda_t(\theta) = v_t\big(\Lambda_t(\theta)\big), \qquad \Lambda_0(\theta) = \theta,
\label{app:cnf-ode}
\end{equation}
has a unique global solution for every initial value $\theta$, and the resulting flow $\Lambda_t: \mathbb{R}^{d_\theta} \to \mathbb{R}^{d_\theta}$ is a diffeomorphism for every $t \in [0, 1]$, whose inverse is obtained by integrating Equation~\eqref{app:cnf-ode} backwards in time. A CNF \citep{chen2018neural} parameterizes $v_t$ with a neural network and defines the family of distributions
\begin{equation}
q_t = (\Lambda_t)_{\#}\, q_0,
\label{app:cnf-pushforward}
\end{equation}
i.e., $q_t$ is the distribution of $\ttt = \Lambda_t(\ttz)$ for $\ttz \sim q_0$. The following two results characterize how $q_t$ evolves in time.

\begin{proposition}[Continuity equation]
The densities $q_t$ in Equation~\eqref{app:cnf-pushforward} satisfy
\begin{equation}
\frac{\partial}{\partial t} q_t(\theta) = -\mathrm{div}\big(q_t(\theta)\, v_t(\theta)\big).
\label{app:continuity-equation}
\end{equation}
\end{proposition}
\begin{proof}
Let $\varphi: \mathbb{R}^{d_\theta} \to \mathbb{R}$ be a smooth test function with compact support. By definition of the pushforward, $\mathbb{E}_{q_t}[\varphi(\theta)] = \mathbb{E}_{q_0}[\varphi(\Lambda_t(\theta))]$. Differentiating both sides with respect to $t$ and using Equation~\eqref{app:cnf-ode} yields
\begin{align*}
\int \varphi(\theta)\, \frac{\partial q_t(\theta)}{\partial t}\, \mathrm{d}\theta
&= \mathbb{E}_{q_0}\Big[\nabla\varphi\big(\Lambda_t(\theta)\big)^\top v_t\big(\Lambda_t(\theta)\big)\Big]
= \int \nabla\varphi(\theta)^\top v_t(\theta)\, q_t(\theta)\, \mathrm{d}\theta \\
&= -\int \varphi(\theta)\, \mathrm{div}\big(q_t(\theta)\, v_t(\theta)\big)\, \mathrm{d}\theta,
\end{align*}
where the last step is integration by parts, whose boundary term vanishes since $\varphi$ has compact support. Since the identity holds for all test functions $\varphi$, Equation~\eqref{app:continuity-equation} follows.
\end{proof}

\begin{proposition}[Instantaneous change of variables]
Along a trajectory $\ttt = \Lambda_t(\ttz)$, the log-density of the CNF evolves as
\begin{equation}
\frac{\mathrm{d}}{\mathrm{d}t} \log q_t\big(\ttt\big) = -\mathrm{div}\, v_t\big(\ttt\big),
\label{app:instantaneous-cov}
\end{equation}
such that
\begin{equation}
\log q_1\big(\tto\big) = \log q_0\big(\ttz\big) - \int_0^1 \mathrm{div}\, v_t\big(\ttt\big)\, \mathrm{d}t,
\label{app:cnf-log-density}
\end{equation}
which is Equation~\eqref{eqn:cnf} of the main text in log-space.
\end{proposition}
\begin{proof}
By the chain rule and Equation~\eqref{app:cnf-ode},
\[
\frac{\mathrm{d}}{\mathrm{d}t} \log q_t\big(\ttt\big) = \frac{\partial \log q_t}{\partial t}\big(\ttt\big) + \nabla \log q_t\big(\ttt\big)^\top v_t\big(\ttt\big).
\]
Expanding the divergence in Equation~\eqref{app:continuity-equation} as $\mathrm{div}(q_t v_t) = q_t\, \mathrm{div}\, v_t + \nabla q_t^\top v_t$ and dividing by $q_t$ gives $\partial_t \log q_t = -\mathrm{div}\, v_t - \nabla \log q_t^\top v_t$. Inserting this into the display above, the two gradient terms cancel, which yields Equation~\eqref{app:instantaneous-cov}. Integrating from $t = 0$ to $t = 1$ yields Equation~\eqref{app:cnf-log-density}.
\end{proof}

Equation~\eqref{app:cnf-log-density} is the continuous-time analogue of the change-of-variables formula for discrete flows (Equation~\eqref{app:discrete-cov} below). To see this, let $J_t(\theta) = \partial \Lambda_t(\theta) / \partial \theta$. Differentiating Equation~\eqref{app:cnf-ode} with respect to $\theta$ gives the variational equation
\[
\frac{\mathrm{d}}{\mathrm{d}t} J_t(\theta) = \nabla_\theta v_t\big(\Lambda_t(\theta)\big)\, J_t(\theta), \qquad J_0(\theta) = I.
\]
By Jacobi's formula,
\[
\frac{\mathrm{d}}{\mathrm{d}t} \log \big|\det J_t(\theta)\big|
= \operatorname{tr}\Big(J_t(\theta)^{-1}\, \frac{\mathrm{d}}{\mathrm{d}t} J_t(\theta)\Big)
= \operatorname{tr}\Big(\nabla_\theta v_t\big(\Lambda_t(\theta)\big)\Big)
= \mathrm{div}\, v_t\big(\Lambda_t(\theta)\big),
\]
such that $\log |\det J_1(\ttz)| = \int_0^1 \mathrm{div}\, v_t(\ttt)\, \mathrm{d}t$, and Equation~\eqref{app:cnf-log-density} is exactly $q_1(\tto) = q_0(\ttz)\, |\det J_1(\ttz)|^{-1}$.

Maximum likelihood training of a CNF requires evaluating Equation~\eqref{app:cnf-log-density} for every training sample in every gradient step, i.e., solving the augmented ODE for $(\ttt, \log q_t(\ttt))$ and computing $\mathrm{div}\, v_t$, which costs $d_\theta$ vector--Jacobian products or a stochastic trace estimate per function evaluation. Flow matching avoids this entirely.

\subsubsection{Flow matching}
\label{app:flow-matching}
Let $\pi_{\mathrm{target}}$ denote the distribution we want to sample from, and let $(\varrho_t)_{t \in [0, 1]}$ be a probability path with $\varrho_0 = q_0$ and $\varrho_1 = \pi_{\mathrm{target}}$, generated by a vector field $u_t$, i.e., $(\varrho_t, u_t)$ satisfy the continuity equation~\eqref{app:continuity-equation}. If $v_t = u_t$, then $q_t = \varrho_t$ for all $t$, since the continuity equation has a unique solution for Lipschitz vector fields. \textit{Flow matching} \citep{lipman2023flow} therefore regresses the neural network $v_t$ onto $u_t$:
\begin{equation}
\hat{\phi} = \arg\min_{\phi}\, \mathcal{L}_{\mathrm{FM}}(\phi), \qquad
\mathcal{L}_{\mathrm{FM}}(\phi) = \mathbb{E}_{t \sim \mathcal{U}(0, 1),\, \theta \sim \varrho_t}\Big[ \big\lVert v_t(\theta) - u_t(\theta) \big\rVert^2 \Big].
\label{app:marginal-objective}
\end{equation}
For any nontrivial target, neither $\varrho_t$ nor $u_t$ is available in closed form. \citet{lipman2023flow} instead construct both as mixtures of simple \textit{conditional} paths. Let $z \sim \pi(z)$ be a conditioning variable, and let $\varrho_t(\theta | z)$ be conditional probability paths generated by conditional vector fields $u_t(\theta | z)$. We define the marginal path and the marginal vector field as
\begin{align}
\varrho_t(\theta) &= \int \varrho_t(\theta | z)\, \pi(z)\, \mathrm{d}z,
\label{app:marginal-density-path} \\
u_t(\theta) &= \int u_t(\theta | z)\, \frac{\varrho_t(\theta | z)\, \pi(z)}{\varrho_t(\theta)}\, \mathrm{d}z
= \mathbb{E}\big[ u_t(\theta | z) \,\big|\, \theta \big],
\label{app:marginal-vector-field}
\end{align}
where the conditional expectation is taken with respect to the posterior $\varrho_t(z | \theta) \propto \varrho_t(\theta | z)\, \pi(z)$.

\begin{theorem}[\citealp{lipman2023flow}]
\label{app:thm-marginal-vf}
If the conditional vector fields $u_t(\theta | z)$ generate the conditional probability paths $\varrho_t(\theta | z)$, then the marginal vector field $u_t$ in Equation~\eqref{app:marginal-vector-field} generates the marginal probability path $\varrho_t$ in Equation~\eqref{app:marginal-density-path}, for any distribution $\pi(z)$.
\end{theorem}
\begin{proof}
It suffices to show that $(\varrho_t, u_t)$ satisfy the continuity equation. Using that each conditional pair satisfies it,
\begin{align*}
\frac{\partial}{\partial t} \varrho_t(\theta)
&= \int \frac{\partial}{\partial t} \varrho_t(\theta | z)\, \pi(z)\, \mathrm{d}z
= -\int \mathrm{div}\big( u_t(\theta | z)\, \varrho_t(\theta | z) \big)\, \pi(z)\, \mathrm{d}z \\
&= -\mathrm{div}\Big( \int u_t(\theta | z)\, \varrho_t(\theta | z)\, \pi(z)\, \mathrm{d}z \Big)
= -\mathrm{div}\big( u_t(\theta)\, \varrho_t(\theta) \big),
\end{align*}
where the last equality is Equation~\eqref{app:marginal-vector-field} multiplied by $\varrho_t(\theta)$.
\end{proof}

Theorem~\ref{app:thm-marginal-vf} shows that $u_t$ is the right regression target, but evaluating it still requires the intractable integral in Equation~\eqref{app:marginal-vector-field}. The key observation of \citet{lipman2023flow} is that we can regress onto the conditional vector field instead:
\begin{equation}
\mathcal{L}_{\mathrm{CFM}}(\phi) = \mathbb{E}_{t \sim \mathcal{U}(0, 1),\, z \sim \pi(z),\, \theta \sim \varrho_t(\theta | z)}\Big[ \big\lVert v_t(\theta) - u_t(\theta | z) \big\rVert^2 \Big],
\label{app:conditional-objective}
\end{equation}
which has the same gradients with respect to $\phi$ as Equation~\eqref{app:marginal-objective} \citep[Theorem~2]{lipman2023flow}. The objective in Equation~\eqref{app:conditional-objective} is simulation-free: it only requires samples from $\pi(z)$ and $\varrho_t(\theta | z)$, and evaluations of $u_t(\theta | z)$.

\paragraph{Linear interpolation paths.} Following \citet{liu2023flow}, we condition on both endpoints, $z = (\ttz, \tto) \sim q_0 \otimes \pi_{\mathrm{target}}$, and connect them by a straight line:
\begin{equation}
\begin{split}
\varrho_t(\theta | \ttz, \tto) &= \delta\big(\theta - \psi_t(\ttz, \tto)\big), \qquad \psi_t(\ttz, \tto) = t\,\tto + (1 - t)\,\ttz, \\
u_t(\theta | \ttz, \tto) &= \tto - \ttz.
\end{split}
\label{app:rectified-map-liu}
\end{equation}
Here, $u_t(\cdot | \ttz, \tto)$ generates $\varrho_t(\cdot | \ttz, \tto)$ in the sense that $\frac{\mathrm{d}}{\mathrm{d}t} \psi_t = \tto - \ttz$. As in Equation~\eqref{eqn:fm-path} of the main text, we write $\ttt = \psi_t(\ttz, \tto)$ for the interpolant from here on; note that it is in general not the trajectory $\Lambda_t(\ttz)$ of the learned flow.

\begin{proposition}
\label{app:prop-linear-path}
Let $q_0$ have a density. The linear interpolation paths in Equation~\eqref{app:rectified-map-liu} satisfy:
\begin{enumerate}
    \item $\varrho_0 = q_0$ and $\varrho_1 = \pi_{\mathrm{target}}$ exactly, for an arbitrary base distribution $q_0$.
    \item For $t \in [0, 1)$, the marginal vector field is
    \begin{equation}
    u_t(\theta) = \mathbb{E}\big[\tto - \ttz \,\big|\, \ttt = \theta\big] = \frac{\mathbb{E}\big[\tto \,\big|\, \ttt = \theta\big] - \theta}{1 - t}.
    \label{app:linear-marginal-vf}
    \end{equation}
    \item If $q_0 = \mathcal{N}(0, I)$, the marginal path $\varrho_t$ coincides with the optimal transport path of \citet{lipman2023flow} with $\sigma_\mathrm{min} = 0$.
\end{enumerate}
\end{proposition}
\begin{proof}
(i) follows from $\psi_0(\ttz, \tto) = \ttz$ and $\psi_1(\ttz, \tto) = \tto$. For (ii), the first equality is Equation~\eqref{app:marginal-vector-field}. On the event $\ttt = \theta$, we have $\ttz = (\theta - t\,\tto)/(1 - t)$, hence $\tto - \ttz = (\tto - \theta)/(1 - t)$, and taking the conditional expectation gives the second equality. For (iii), conditionally on $\tto$, the interpolant $\ttt = t\,\tto + (1 - t)\,\ttz$ with $\ttz \sim \mathcal{N}(0, I)$ is distributed as $\mathcal{N}(t\,\tto, (1 - t)^2 I)$, which is the conditional path $\mathcal{N}(t\,\tto, (1 - (1 - \sigma_\mathrm{min})t)^2 I)$ of \citet{lipman2023flow} with $\sigma_\mathrm{min} = 0$. Marginalizing over the same $\pi_{\mathrm{target}}$ yields the same $\varrho_t$.
\end{proof}

Property (i) is what allows CPE to use the prior as base distribution (Appendix~\ref{app:npe}). While every conditional path is straight, the marginal flow generated by Equation~\eqref{app:linear-marginal-vf} is in general curved, since conditional paths belonging to different endpoint pairs intersect. Its trajectories nevertheless never cross, by uniqueness of ODE solutions \citep{liu2023flow}. Combining Equations~\eqref{app:conditional-objective} and~\eqref{app:rectified-map-liu} yields the objective
\begin{equation}
\mathcal{L}(\phi) = \mathbb{E}_{t \sim \mathcal{U}(0, 1),\, \ttz \sim q_0,\, \tto \sim \pi_{\mathrm{target}}}\Big[ \big\lVert v_t\big(\ttt\big) - \big(\tto - \ttz\big) \big\rVert^2 \Big].
\label{app:linear-objective}
\end{equation}

\subsubsection{Neural posterior estimation}
\label{app:npe}
For neural posterior estimation, the target distribution is the posterior, $\pi_{\mathrm{target}} = \pi(\theta | x)$, for every observation $x$. We therefore make the vector field conditional on the data, $v: [0, 1] \times \mathbb{R}^{d_\theta} \times \mathbb{R}^{d_x} \to \mathbb{R}^{d_\theta}$, $(t, \theta, x) \mapsto v_t(\theta, x)$. For fixed $x$, Proposition~\ref{app:prop-linear-path} gives the target vector field
\[
u_t(\theta | x) = \mathbb{E}\big[\tto - \ttz \,\big|\, \ttt = \theta,\, x\big], \qquad \ttz \sim q_0,\ \tto \sim \pi(\theta | x),
\]
whose flow transports $q_0$ to $\pi(\theta | x)$. To amortize over observations, we average the objective in Equation~\eqref{app:linear-objective} over the evidence $\pi(x)$,
\[
\mathcal{L}_{\mathrm{NPE}}(\phi) = \mathbb{E}_{\pi(x)}\, \mathbb{E}_{t \sim \mathcal{U}(0, 1),\, \ttz \sim q_0,\, \tto \sim \pi(\theta | x)}\Big[ \big\lVert v_t\big(\ttt, x\big) - \big(\tto - \ttz\big) \big\rVert^2 \Big].
\]
Neither $\pi(x)$ nor $\pi(\theta | x)$ can be sampled directly, but by Bayes' theorem $\pi(x)\, \pi(\theta | x) = \pi(\theta)\, \pi(x | \theta)$, such that
\begin{equation}
\mathcal{L}_{\mathrm{NPE}}(\phi) = \mathbb{E}_{t \sim \mathcal{U}(0, 1),\, \ttz \sim q_0,\, \tto \sim \pi(\theta),\, x \sim \pi(x | \tto)}\Big[ \big\lVert v_t\big(\ttt, x\big) - \big(\tto - \ttz\big) \big\rVert^2 \Big],
\label{app:npe-objective}
\end{equation}
which only requires draws from the prior and the simulator \citep{wildberger2023flow}. This is the objective in Equation~\eqref{eqn:npe-loss} of the main text.

\paragraph{Prior as base distribution.} Since Proposition~\ref{app:prop-linear-path}(i) holds for arbitrary base distributions, CPE uses the prior itself as base, $q_0 = \pi(\theta)$ (Section~\ref{sec:prior-base-distribution}). Both endpoints $\ttz$ and $\tto$ are then prior draws, and the learned flow transports the prior to the posterior. In practice, we sample $t \sim \mathcal{U}(\varepsilon, 1)$ with $\varepsilon = 10^{-3}$ for numerical stability, and integrate the ODE from $t = \varepsilon$ to $t = 1$. Since the ODE is not constrained to the support of the prior, a sample can leave it, in which case we discard it (Section~\ref{sec:training-sampling}).

\subsubsection{Architecture}
\label{app:cnf-architecture}
In the continuous case, the vector field is a function of time, parameters and data, $v: [0, 1] \times \mathbb{R}^{d_\theta} \times \mathbb{R}^{d_x} \to \mathbb{R}^{d_\theta}$. Since a CNF does not require the network to be invertible, the activation functions $\mathrm{act}$ and $g$ (Section~\ref{sec:cond-dep}) can be arbitrary; we use $\mathrm{act} = \tanh$ and the identity for $g$. See Figure~\ref{app:cpe-architecture} for an overview. The structural masks enforce the dependency structure of the posterior program on the vector field, which we formalize below.

Recall the topological ordering $\omega = \{\omega_1, \dots, \omega_n\}$ of the $n$ parameter nodes of the posterior program (Section~\ref{sec:cond-dep}), where a node $\theta_{\omega_i}$ may be multi-dimensional, and denote by $\mathrm{pa}(\theta_{\omega_i})$ and $\mathrm{an}(\theta_{\omega_i})$ its parents and ancestors in the posterior program.

\begin{lemma}
\label{app:lemma-sparsity}
For all $t$, $x$ and $i, j \in \{1, \dots, n\}$,
\[
\frac{\partial\, v_t(\theta, x)_{\omega_i}}{\partial\, \theta_{\omega_j}} = 0 \qquad \text{unless } j = i \text{ or } \theta_{\omega_j} \in \mathrm{an}(\theta_{\omega_i}).
\]
In particular, $\nabla_\theta v_t(\theta, x)$ is block lower-triangular with respect to $\omega$.
\end{lemma}
\begin{proof}
Consider a single block $g^{(l)}_{t, x}$ and denote by $h^{(k)}_{\omega_i}$ the hidden units of node $\theta_{\omega_i}$ after the $k$-th layer $f^{(k)}$. The structural mask of $f^{(k)}$ connects the output units of $\theta_{\omega_i}$ only to the input units of $\theta_{\omega_i}$ and its parents $\mathrm{pa}(\theta_{\omega_i})$. The activation acts elementwise, the reversal of hidden units permutes units within a node, and the linear projection of the conditioning vector does not depend on $\theta$. By induction over $k$, $h^{(k)}_{\omega_i}$ depends on $\theta_{\omega_j}$ only if $j = i$ or $\theta_{\omega_j}$ is an ancestor of $\theta_{\omega_i}$ at distance at most $k$. The convex combination in Equation~\eqref{eqn:convex-combination} adds $\alpha_l$ times the identity, which preserves this structure, and so does the composition of $L$ blocks in Equation~\eqref{eqn:vector-field}, since the ancestor relation is transitive.
\end{proof}

Within a node with more than one dimension, the masks are additionally autoregressive over the node's dimensions. In the continuous variant, we reverse the order of the hidden units of each node after every intermediate layer $f^{(2)}, \dots, f^{(K-1)}$, which alternates the autoregressive order and lets all dimensions of a node depend on each other. By Lemma~\ref{app:lemma-sparsity}, this does not affect the dependencies between nodes.

\begin{figure}[h!]
    \centering
    \scalebox{0.75}{
\begin{tikzpicture}[
    block/.style={rectangle, rounded corners, minimum width=70pt, minimum height=15pt, draw, thick, font=\ttfamily\small},
    block-transform/.style={block, fill=violet!20, draw=violet!60, line width=0.8pt},
    add-tanh/.style={block, fill=red!20, draw=red!40, line width=0.8pt},
    linear/.style={block, fill=blue!20, draw=blue!40, line width=0.8pt},
    fourier/.style={block, fill=orange!20, draw=orange!40, line width=0.8pt},
    combo/.style={block, fill=green!20, draw=green!40, line width=0.8pt},
    compute-block/.style={rectangle, rounded corners, minimum width=100pt, minimum height=80pt, draw=gray!80, line width=0.8pt},
    param-block/.style={rectangle, rounded corners, minimum width=40pt, minimum height=15pt, draw, thick, line width=0.8pt, font=\ttfamily\small},    
    arr/.style={-{Stealth[length=5pt]}, line width=0.7pt}
]

\node[param-block] (theta) at (-6, 7.5) {$\theta$};
\node[param-block] (x) at (4, 7.5) {$x$};
\node[param-block] (t) at (0, 7.5) {$t$};

\node[fourier] (fourier) at (0, 6) {Sinusoidal embedding};
\node[linear] (linear1) at (0, 5.25) {Linear};
\node[add-tanh] (activate1) at (0, 4.5) {Activate};
\node[linear] (linear2) at (0, 3.75) {Linear};

\node[linear] (linear3) at (4, 5.25) {Linear};
\node[add-tanh] (activate2) at (4, 4.5) {Activate};
\node[linear] (linear4) at (4, 3.75) {Linear};

\node[add-tanh] (concat) at (2, 3) {Concat and activate};

\node[linear] (linear5) at (2, 1.25) {Linear};
\node[linear] (linear6) at (2, -1) {Linear};

\node[block-transform] (block1) at (-3, 2) {Block transform};
\node[add-tanh] (add1) at (-3, 1.25) {Add and tanh};

\node[compute-block] (compute) at (-3, -1) {};
\node[block-transform] (block2) at (-3, -0.25) {Block transform};
\node[add-tanh] (add2) at (-3, -1) {Add and tanh};
\node[block-transform] (block3) at (-3, -1.75) {Block transform};

\node[combo] (combo) at (-6, -3.25) {Convex combination};
\node[anchor=north east] at ([xshift=0pt, yshift=0pt]compute.north west) {Nx};
\node[param-block] (theta_prime) at (-6, -4.75) {$\theta'$};

\draw[arr] (t) -- (fourier);
\draw[arr, -] (fourier) -- (linear1);
\draw[arr, -] (linear1) -- (activate1);
\draw[arr, -] (activate1) -- (linear2);
\draw[arr] (x) -- (linear3);

\draw[arr, -] (linear3) -- (activate2);
\draw[arr, -] (activate2) -- (linear4);

\draw[arr] (linear2) |- ($(concat.west) + (-0.25,0)$) -- (concat.west);
\draw[arr] (linear4) |- ($(concat.east) + (0.25,0)$) -- (concat.east);

\begin{pgfonlayer}{background}
    \draw[arr] (concat) -- (linear5);
   \draw[arr, dashed]  (concat)
    -- ++(0,-1)       
    -- ++(1.5,0)        
    -- ++(0,-2.25)       
    -- ++(-1.5,0)        
    -- (linear6);   
\end{pgfonlayer}

\draw[arr] (linear5) -- (add1);
\draw[arr, dashed] (linear6) -- (add2);

\draw[arr] (theta) -- (combo);
\draw[arr] (theta)  |- ($(block1.north) + (0,0.35)$) --  (block1);

\draw[arr, -] (block1) -- (add1);
\draw[arr] (add1) -- (block2);
\draw[arr, -] (block2) -- (add2);
\draw[arr, -] (add2) -- (block3);
\draw[arr, -] (block3) |- ($(combo.east) + (0,0)$) -- (combo.east);
\draw[arr] (combo) -- (theta_prime);
\end{tikzpicture}
}
\caption{Continuous-time CPE architecture. The data $x$ and the time $t$ are embedded using an MLP and a sinusoidal embedding followed by an MLP, respectively, and concatenated to form the conditioning vector $c$. We add a linear projection of $c$ after the first block transform, before applying the activation function (Equation~\eqref{eqn:linear-projection}). While it is possible to condition after every block transform (dashed lines), we only do so after the first. Finally, the output is combined with the input through the convex combination in Equation~\eqref{eqn:convex-combination}.}
\label{app:cpe-architecture}
\end{figure}

\subsection{CPE with discrete normalizing flows}
\label{app:cpe-nf}

\subsubsection{Discrete normalizing flows}
Like CNFs, discrete normalizing flows define a distribution as the pushforward $q = \Lambda_{\#}\, q_0$ of the base distribution under a diffeomorphism $\Lambda: \mathbb{R}^{d_\theta} \to \mathbb{R}^{d_\theta}$, but parameterize $\Lambda$ directly instead of through an ODE. For a measurable set $\mathcal{A}$, substituting $\ttz = \Lambda^{-1}(\theta)$ gives
\[
\mathbb{P}\big(\Lambda(\ttz) \in \mathcal{A}\big)
= \int_{\Lambda^{-1}(\mathcal{A})} q_0\big(\ttz\big)\, \mathrm{d}\ttz
= \int_{\mathcal{A}} q_0\big(\Lambda^{-1}(\theta)\big) \left| \det \frac{\partial \Lambda^{-1}(\theta)}{\partial \theta} \right| \mathrm{d}\theta,
\]
and hence the change-of-variables formula
\begin{equation}
q(\theta) = q_0\big(\Lambda^{-1}(\theta)\big) \left| \det \frac{\partial \Lambda^{-1}(\theta)}{\partial \theta} \right|.
\label{app:discrete-cov}
\end{equation}
If $\Lambda^{-1} = T_M \circ \dots \circ T_1$ is a composition, the chain rule and the multiplicativity of the determinant give
\[
\log \left| \det \frac{\partial \Lambda^{-1}(\theta)}{\partial \theta} \right| = \sum_{m=1}^{M} \log \left| \det \frac{\partial T_m(\eta_{m-1})}{\partial \eta_{m-1}} \right|, \qquad \eta_0 = \theta,\ \eta_m = T_m(\eta_{m-1}).
\]
Evaluating a general determinant costs $\mathcal{O}(d_\theta^3)$. If the Jacobian is triangular, however, the determinant is the product of its diagonal entries and costs $\mathcal{O}(d_\theta)$. Autoregressive flows, to which the block transforms of CPE belong, have triangular Jacobians by construction.

\subsubsection{Invertible block transforms}
\label{app:invertible-block-transforms}
We now make the block transforms of Section~\ref{sec:cond-dep} invertible, following the block neural autoregressive flow of \citet{cao2020block}. Order the coordinates of $\theta$ by $\omega$, let $h^{(0)} = \theta$, and for $k = 1, \dots, K$ let
\[
\tilde{h}^{(k)} = B^{(k)} h^{(k-1)} + b^{(k)}, \qquad
h^{(k)} = \begin{cases} \mathrm{act}\big(\tilde{h}^{(k)}\big), & k < K, \\ \tilde{h}^{(k)}, & k = K, \end{cases}
\]
i.e., Equation~\eqref{eqn:linear-projection} without the activation in the last layer. In the conditional case, the linear projection $\texttt{Linear}(c)$ of the conditioning vector $c$ is added to $\tilde{h}^{(1)}$. Each coordinate $\theta_i$ owns a block of $d^{(k)}_{\mathrm{out}}$ hidden units in layer $k$, with $d^{(k)}_{\mathrm{in}} = d^{(k-1)}_{\mathrm{out}}$ and $d^{(1)}_{\mathrm{in}} = d^{(K)}_{\mathrm{out}} = 1$, such that $B^{(k)}$ is a $d_\theta \times d_\theta$ grid of blocks $B^{(k)}_{ij} \in \mathbb{R}^{d^{(k)}_{\mathrm{out}} \times d^{(k)}_{\mathrm{in}}}$. The structural mask makes $B^{(k)}$ block lower-triangular (Figure~\ref{fig:method-mask}), and the diagonal blocks are $B^{(k)}_{ii} = g(\bar{B}^{(k)}_{ii})$, where $g = \exp$ is applied elementwise to unconstrained weights $\bar{B}^{(k)}_{ii}$.

\begin{proposition}
\label{app:prop-bnaf-logdet}
Let $\mathrm{act}$ be strictly increasing and differentiable. Then $\lambda = h^{(K)}$ has a lower-triangular Jacobian with strictly positive diagonal entries
\begin{equation}
\frac{\partial \lambda_i}{\partial \theta_i} = B^{(K)}_{ii}\, D^{(K-1)}_i\, B^{(K-1)}_{ii} \cdots D^{(1)}_i\, B^{(1)}_{ii},
\qquad D^{(k)}_i = \mathrm{diag}\big(\mathrm{act}'(\tilde{h}^{(k)}_i)\big),
\label{app:bnaf-diagonal}
\end{equation}
where $\tilde{h}^{(k)}_i \in \mathbb{R}^{d^{(k)}_{\mathrm{out}}}$ are the pre-activations of the block of coordinate $\theta_i$. In particular, $\theta_i \mapsto \lambda_i(\theta)$ is strictly increasing for fixed preceding coordinates.
\end{proposition}
\begin{proof}
By the chain rule, $\nabla_\theta \lambda = B^{(K)} D^{(K-1)} B^{(K-1)} \cdots D^{(1)} B^{(1)}$, where $D^{(k)} = \mathrm{diag}(\mathrm{act}'(\tilde{h}^{(k)}))$ is block-diagonal and each $B^{(k)}$ is block lower-triangular. Products of block lower-triangular matrices are block lower-triangular, and the diagonal blocks of the product are the products of the diagonal blocks, which gives Equation~\eqref{app:bnaf-diagonal}. Since $d^{(1)}_{\mathrm{in}} = d^{(K)}_{\mathrm{out}} = 1$, the right-hand side is a scalar. All entries of $B^{(k)}_{ii} = \exp(\bar{B}^{(k)}_{ii})$ and $D^{(k)}_i$ are strictly positive, so the product is strictly positive.
\end{proof}

Since the conditioning term does not depend on $\theta$, it does not contribute to the Jacobian. To evaluate Equation~\eqref{app:bnaf-diagonal} stably, we compute it in log-space. Let
\[
\big[\operatorname{LME}(P, Q)\big]_{jl} = \log \sum_{m} \exp\big(P_{jm} + Q_{ml}\big),
\]
such that $\log(PQ) = \operatorname{LME}(\log P, \log Q)$ for matrices with positive entries. With $\log B^{(k)}_{ii} = \bar{B}^{(k)}_{ii}$, the recursion
\begin{align*}
\ell^{(1)}_i &= \bar{B}^{(1)}_{ii} + \log \mathrm{act}'\big(\tilde{h}^{(1)}_i\big), \\
\ell^{(k)}_i &= \operatorname{LME}\big(\bar{B}^{(k)}_{ii},\, \ell^{(k-1)}_i\big) + \log \mathrm{act}'\big(\tilde{h}^{(k)}_i\big), \qquad k = 2, \dots, K-1, \\
\log \frac{\partial \lambda_i}{\partial \theta_i} &= \operatorname{LME}\big(\bar{B}^{(K)}_{ii},\, \ell^{(K-1)}_i\big),
\end{align*}
where $\log \mathrm{act}'(\tilde{h}^{(k)}_i)$ is added to every column, yields $\log |\det \nabla_\theta \lambda| = \sum_{i=1}^{d_\theta} \log \partial \lambda_i / \partial \theta_i$. For $\mathrm{act} = \tanh$, the log-derivative can be computed without cancellation as
\[
\log \tanh'(s) = \log\big(1 - \tanh^2(s)\big) = 2\big(\log 2 - s - \mathrm{softplus}(-2s)\big).
\]

\paragraph{Convex combination.} The prior-guided block $g^{(l)}_x(u) = \alpha_l\, u + (1 - \alpha_l)\, \lambda^{(l)}(u, x)$ from Equation~\eqref{eqn:convex-combination} has Jacobian $\alpha_l I + (1 - \alpha_l)\, \nabla_u \lambda^{(l)}$, which is again lower-triangular with diagonal entries $\alpha_l + (1 - \alpha_l)\, \partial \lambda^{(l)}_i / \partial u_i > 0$. Parameterizing $\alpha_l = 1 - \sigma(\kappa_l) = 1/(1 + e^{\kappa_l})$ with an unconstrained $\kappa_l \in \mathbb{R}$ and writing $\ell_i = \log \partial \lambda^{(l)}_i / \partial u_i$,
\[
\log \frac{\partial g^{(l)}_{x, i}}{\partial u_i}
= \log \frac{1 + e^{\kappa_l + \ell_i}}{1 + e^{\kappa_l}}
= \mathrm{softplus}(\kappa_l + \ell_i) - \mathrm{softplus}(\kappa_l).
\]
The convex combination also makes the block surjective. Since $\tanh$ is bounded and the last layer is affine, each $\lambda^{(l)}_i$ is bounded, say $|\lambda^{(l)}_i| \leq M_i$. Hence $u_i \mapsto g^{(l)}_{x, i}(u)$ is strictly increasing and tends to $\pm\infty$ as $u_i \to \pm\infty$, and by induction over the coordinates, $g^{(l)}_x$ is a bijection of $\mathbb{R}^{d_\theta}$. Without the convex combination ($\alpha_l = 0$), the image of $\lambda^{(l)}$ would be bounded, and Equation~\eqref{app:discrete-cov} would not define a normalized density for base distributions with unbounded support.

\paragraph{No reversal.} Reversing the hidden units of a multi-dimensional node, as in the continuous variant, would break the triangular structure within that node, and Proposition~\ref{app:prop-bnaf-logdet} would no longer apply. The discrete variant therefore omits the reversal.

\subsubsection{Neural posterior estimation}
For posterior estimation, we condition the flow on the data and compose $L$ prior-guided blocks, $\Lambda_x^{-1} = g^{(L)}_x \circ \dots \circ g^{(1)}_x$, which defines
\begin{equation}
q(\theta | x) = q_0\big(\Lambda_x^{-1}(\theta)\big) \left| \det \frac{\partial \Lambda_x^{-1}(\theta)}{\partial \theta} \right|.
\end{equation}
In contrast to the continuous case, evaluating $q(\theta | x)$ does not require solving an ODE, so we can train by maximum likelihood, following the discrete NPE literature \citep[e.g.,][]{greenberg2019automatic}:
\begin{equation}
\hat{\phi} = \arg\max_{\phi}\, \mathbb{E}_{\theta \sim \pi(\theta),\, x \sim \pi(x | \theta)}\big[ \log q(\theta | x) \big].
\end{equation}
Using $\pi(\theta)\, \pi(x | \theta) = \pi(x)\, \pi(\theta | x)$,
\[
\mathbb{E}_{\pi(\theta, x)}\big[\log q(\theta | x)\big]
= -\mathbb{E}_{\pi(x)}\Big[ \mathrm{KL}\big( \pi(\theta | x) \,\big\|\, q(\theta | x) \big) \Big] + \mathbb{E}_{\pi(\theta, x)}\big[\log \pi(\theta | x)\big],
\]
where the second term does not depend on $\phi$. Maximum likelihood training therefore minimizes the expected forward KL divergence to the posterior.

\paragraph{Sampling.} To sample, we draw $\ttz \sim q_0$ and compute $\theta = \Lambda_x(\ttz)$ by inverting the blocks $g^{(L)}_x, \dots, g^{(1)}_x$ in turn. A block has no analytical inverse, but its triangular structure allows inverting it one coordinate at a time, in the order $\omega$: to solve $g^{(l)}_x(u) = w$ for $u$, given $u_1, \dots, u_{i-1}$, the coordinate $u_i$ solves $g^{(l)}_{x, i}(u_1, \dots, u_i) = w_i$, where the left-hand side is strictly increasing in $u_i$. The bound $|\lambda^{(l)}_i| \leq M_i$ brackets the solution,
\[
u_i \in \left[ \frac{w_i - (1 - \alpha_l) M_i}{\alpha_l},\ \frac{w_i + (1 - \alpha_l) M_i}{\alpha_l} \right],
\]
such that bisection converges to precision $\epsilon$ in $\mathcal{O}\big(\log_2(M_i / (\alpha_l\, \epsilon))\big)$ iterations per coordinate.

\subsubsection{Architecture}
In the discrete case, we drop the time argument, and each block $g^{(l)}_x$ is a transform $\mathbb{R}^{d_\theta} \to \mathbb{R}^{d_\theta}$ that is invertible for every $x$. In Figure~\ref{app:cpe-architecture}, this corresponds to removing the $t$-branch, such that the conditioning vector $c$ consists of the data embedding only. We use $g = \exp$ and $\mathrm{act} = \tanh$, and omit the reversal of hidden units.

\section{Extended related work}
\label{app:related-work}
\paragraph{Simulation-based inference}
Beyond the methods discussed in Section~\ref{sec:related}, continuous-time SBI methods have been extended to lower-variance score matching objectives \citep{ko2026latent,jiang2026likelihood}. \citet{nautiyal2026oneflow} train a single flow matching model over the joint distribution with query-aware masking, which supports posterior sampling, likelihood estimation and arbitrary conditionals with one model.

\paragraph{Causal flows}
A parallel line of work studies normalizing flows whose architecture reflects a structural causal model. \citet{javaloy2023causal} establish a correspondence between autoregressive flows and identifiable causal models via triangular monotonically increasing maps, which enables interventional and counterfactual queries, and \citet{visentin2025robust} enforce causal constraints through hypercoupling flows that transform each variable conditioned on its parents in a given DAG. We emphasize that our use of the term \emph{causal} refers to the conditional independence structure induced by the prior and posterior programs, and that we make no identifiability or interventional claims. Adapting CPE to answer counterfactual queries is an interesting direction for future work.

\paragraph{Robust SBI}
An important but orthogonal line of research develops \textit{robust} SBI methods that account for model misspecification and poor calibration \citep{frazier2020model,hermans2022a,ward2022robust,dellaporta22robust,huang2023learning,gloeckler23adversarial,kelly2024misspecificationrobust,verma2025robust,yuyan2025robust,tomaselli2025robust,ruhlmann2025flow,bharti2026amortised,khoo2026minimum,kelly2026preconditioned,laplante2026conservative}. These methods modify the training objective \citep{laplante2026conservative,bharti2026amortised}, learn robust summary statistics \citep{huang2023learning,khoo2026minimum}, or apply a post-hoc correction to an already trained estimator \citep{ruhlmann2025flow,kelly2026preconditioned}, and are for the most part agnostic to the architecture of the inference network. CPE instead changes the inductive bias of that architecture under a correctly specified simulator. The two are therefore complementary, and applying robust objectives to CPE is a natural extension.
\section{Experimental models}
\label{app:experimental-models}

We evaluate CPE on benchmark models from the recent literature \citep{papamakarios2019sequential, greenberg2019automatic,lueckmann2021benchmarking,forbes2022summary,vargas2024transport,gloeckler2024allinone} and develop new benchmark tasks. We describe the different models below.

\subsection{Linear Gaussian}
The linear Gaussian benchmark task is defined by the following generative process:
\begin{align}
\theta & \sim \mathcal{N}_{10}(0, \sigma^2 I) \\
x \mid \theta & \sim \mathcal{N}_{10}(\theta, \sigma^2 I) 
\end{align}
where $\sigma^2 = 0.1$, $I$ is a unit matrix of appropriate dimensionality, and both $\theta \in \mathbb{R}^{10}$ and $x \in \mathbb{R}^{10}$ are ten-dimensional random variables. The linear Gaussian follows the representation in \citet{lueckmann2021benchmarking}.

\subsection{Gaussian mixture 1}
The Gaussian mixture 1 has the following generative process:
\begin{align*}
\theta & \sim \mathcal{U}_2(-10, 10) \\
 x \mid \theta & \sim \frac{1}{2} \mathcal{N}_2(\theta, I) + \frac{1}{2} \mathcal{N}_2( \theta, \sigma^2  I) 
\end{align*}
where $\sigma^2 = 0.01$, $I$ is a unit matrix, and both $\theta \in \mathbb{R}^2$ and $x \in \mathbb{R}^2$ are two-dimensional random variables. The Gaussian mixture follows the representation in \citet{lueckmann2021benchmarking}.

\subsection{Gaussian mixture 2}
The Gaussian mixture 2 has been proposed in \citet{vargas2024transport}. It is a 3-component mixture that we adapted for SBI. It has the following generative process:
\begin{align*}
\theta & \sim \mathcal{N}_6\left( [-1, -1, 0, 0, 1, 1]^T, I \right) \\
 x \mid \theta & \sim 
 \frac{1}{3} \mathcal{N}_2\left(
 \theta_{1,2}, 
 \begin{bmatrix}
     0.7 & 0 \\
     0 & 0.05
 \end{bmatrix}
 \right) 
 + \frac{1}{3} \mathcal{N}_2\left( \theta_{3,4},
  \begin{bmatrix}
     0.7 & 0 \\
     0 & 0.05
 \end{bmatrix}
 \right) 
 + \frac{1}{3} \mathcal{N}_2\left( \theta_{5,6},
  \begin{bmatrix}
     1 & 0.95 \\
     0.95 & 1
 \end{bmatrix}
 \right) 
\end{align*}
where $I$ is a unit matrix, and $\theta \in \mathbb{R}^6$ and $x \in \mathbb{R}^2$ are six- and two-dimensional random variables, respectively.

\subsection{Nested model}
We propose a nested model for benchmarks in SBI. The model has the following generative process:
\begin{equation*}
\begin{aligned}
\gamma &\sim \mathcal{N}(0, 1) \\
\beta_k \mid \gamma &\sim \mathcal{N}(\gamma, 1), \quad k = 1, 2, 3 \\
x_k \mid \beta_k &\sim \mathcal{N}(\beta_k, 0.1^2), \quad k = 1, 2, 3
\end{aligned}
\end{equation*}
where $\theta = (\gamma, \beta_1, \beta_2, \beta_3)$ is a four-dimensional random variable for which the posterior should be inferred, and $x = (x_1, x_2, x_3) \in \mathbb{R}^3$.

\subsection{Hyperboloid}
The hyperboloid model \citep{forbes2022summary} is a 2-component mixture of Student's $t$-distributions of the form
\begin{align}
\theta &\sim \mathcal{U}_2(-2, 2) \\
x \mid \theta &\sim 
    \frac{1}{2} t_{10}(\nu, F(\theta;  a_1,  a_2) \mathbb{I}, \sigma^2  I) +
    \frac{1}{2} t_{10}(\nu, F(\theta;  b_1,  b_2) \mathbb{I}, \sigma^2  I)
\end{align}
which are parameterized by degrees of freedom $\nu$, mean
\[
F(\theta; m_1, m_2) = \big| \|\theta - m_1\|_2 - \|\theta - m_2\|_2 \big|,
\]
and scale matrix $\sigma^2  I$. 
$\mathbb{I}$ is a ten-dimensional vector of ones. We follow \citet{forbes2022summary} and set $ a_1 = [-0.5, 0.0]^T$, $ a_2 = [0.5, 0.0]^T$, $ b_1 = [0.0, -0.5]^T$, $ b_2 = [0.0, 0.5]^T$, $\nu = 3$ and $\sigma^2 = 0.01$ for our experiments. 

\subsection{Mixture model with distractors}
Similarly to the SLCP task \citep{lueckmann2021benchmarking}, we evaluate a benchmark task that uses distractor dimensions appending the data with dimensions that are not informative of the parameters. The generative process of the model is taken from \citet{albert2025simulated} as follows:
\begin{align*}
    \theta & \sim \mathcal{U}(-10, 10) \\
    x_j \mid \theta &\sim \alpha \mathcal{N}(\theta, 1)
    + (1 - \alpha) \mathcal{N}(-\theta, \sigma^2), \quad j = 1, 2 \\
    x_j &\sim \mathcal{N}(0, 1), \quad j = 3, \dots, 10
\end{align*}
where we set $\alpha = \sigma = 0.3$. In addition to the two informative data dimensions $x_1, x_2$, the model adds $8$ samples from a standard Gaussian that do not carry information of the parameters. If both $x_{1,obs}$ and $x_{2,obs}$ are from the same mode (for one of the observations, we set $x_{1,obs} = x_{2,obs} = 5$), then the posterior for $\theta$ is bimodal with very uneven mass distribution. 

\subsection{Tree}
The tree model is a recently introduced SBI benchmark model \citep{gloeckler2024allinone}. It has the following generative process:
\begin{align*}
\theta_1 &\sim \mathcal{N}(0, 1) \\
\theta_2 \mid \theta_1 &\sim \mathcal{N}(\theta_1, 1) \\
\theta_3 \mid \theta_1 &\sim \mathcal{N}(\theta_1, 1) \\
x_1 \mid \theta_2 &\sim \mathcal{N} \left(\sin(\theta_2)^2, 0.2^2 \right) \\
x_2 \mid \theta_2 &\sim \mathcal{N} \left(0.1\theta_2^2, 0.2^2 \right) \\
x_3 \mid \theta_3 &\sim \mathcal{N} \left(0.1\theta^2_3, 0.6^2 \right) \\
x_4 \mid \theta_3 &\sim \mathcal{N} \left(\cos(\theta_3)^2, 0.1^2 \right)
\end{align*}

\subsection{Two moons}
Two moons is a common benchmark task in the SBI literature \citep{greenberg2019automatic}. Its generative process is defined as:
\begin{align*}
\theta & \sim \mathcal{U}_2(-1, 1) \\
\alpha & \sim \mathcal{U}(-\pi / 2, \pi / 2) \\
r & \sim \mathcal{N}(0.1, 0.01^2) \\
 x \mid \theta & = 
\begin{pmatrix}
    r \cos{\alpha} + 0.25 \\ r \sin{\alpha}
\end{pmatrix}
 +
 \begin{pmatrix}
    -|\theta_1 + \theta_2| / \sqrt{2} \\ (-\theta_1 + \theta_2) / \sqrt{2}
\end{pmatrix}\,,
\end{align*}
where we treat $\alpha$ and $r$ as nuisance parameters.

\subsection{SIR model}
The SIR model \citep{lueckmann2021benchmarking} describes the spread of an infectious disease in a population of $N_\mathrm{pop} = 10^6$ individuals, divided into susceptible ($S$), infected ($I$), and recovered ($R$) compartments, with dynamics
\begin{equation*}
\frac{\mathrm{d}S}{\mathrm{d}t} = -\theta_0 \frac{S I}{N_\mathrm{pop}}, \qquad
\frac{\mathrm{d}I}{\mathrm{d}t} = \theta_0 \frac{S I}{N_\mathrm{pop}} - \theta_1 I, \qquad
\frac{\mathrm{d}R}{\mathrm{d}t} = \theta_1 I.
\end{equation*}
We infer the contact rate $\theta_0 \sim \log\mathcal{N}(\log 0.4, 0.5^2)$ and the recovery rate $\theta_1 \sim \log\mathcal{N}(\log 1/8, 0.2^2)$. Starting from a single infected individual, we solve the ODE on $t \in [0, 160]$ and, at each of $160$ equally spaced time points, observe the number of infected individuals in a sample of $1\,000$ via a binomial distribution. Instead of the full time series, all methods are trained on five summary statistics: the peak number of infections, the time of the peak, and the sum, mean, and variance of the infection counts. These summaries are informative but not sufficient for the parameters. The reference posterior is obtained by slice sampling on the full time series.

\subsection{Hierarchical model}
The hierarchical model draws a two-dimensional top-level location $\gamma$, three lower-level locations $\beta_k$ centered at $\gamma$, and an observation scale $\sigma$, and generates each observation $x_k$ from a Gaussian centered at $\beta_k$ with scale $\sigma$. In total, it has nine parameters and six-dimensional observations.
\begin{equation*}
    \begin{split}    
    \sigma & \sim \text{HalfNormal}(1) \\
    \gamma & \sim \mathcal{N}_2(0, I) \\
    \beta_k \mid \gamma & \sim \mathcal{N}_2(\gamma, I), \quad k = 1, 2, 3 \\
    x \mid \beta, \sigma & \sim \mathcal{N}_6\big([\beta_1^\top, \beta_2^\top, \beta_3^\top]^\top, \sigma^2 I\big)
    \end{split}
\end{equation*}

\section{Experimental details}
\label{app:experimental-details}

In the following, we describe the experimental details for each benchmark task and inferential algorithm. All source code to reproduce our results can be found in the supplemental material and on GitHub (link published after review period).

\subsection{Reference posterior distributions}
We use the Python package \texttt{sbijax} \citep{dirmeier2024simulation} to draw reference, i.e., ``ground-truth'' posterior samples using Markov chain Monte Carlo. For all experimental models, we use a hit-and-run slice sampler with a step size of $1$ and at most $5$ doubling steps, and run $10$ independent chains. For each chain, we discard the first $5\,000$ iterations as warmup and then keep every second of the following $10\,000$ iterations to reduce autocorrelation, yielding $5\,000$ samples per chain. Hence, the reference posterior sample consists of $50\,000$ particles. We monitored convergence using the typical MCMC diagnostic tools, i.e., the rank-normalized split $\widehat{R}$ \citep{vehtari2021rank} and the effective sample size as implemented in the Python package ArviZ \citep{kumar2019arviz}.

\subsection{Comparing reference posteriors to inferred posteriors}
We compare the inferred posterior distributions to the reference posterior distributions using the classifier two-sample test statistic \citep[C2ST;][]{lopezpaz2017revisiting} and the recently introduced H-Min divergence \citep{zhao2022comparing}, which we, like \citet{dirmeier2023simulation}, found to be easier to tune than C2ST and other divergences. To compute the metrics, we first subsample both reference and inferred posteriors to $10\,000$ samples without replacement, and then compute both metrics on these subsamples.

\subsection{Neural network architectures}
For AIO, FMPE and PSE, we use implementations from the Python package \texttt{sbijax} \citep{dirmeier2024simulation}. The neural network architectures of CPE and all baselines are detailed below.

\paragraph{CPE} CPE is structured as a block neural autoregressive flow. We first embed the time $t$ using a sinusoidal embedding of dimension $64$ and then project the result through a two-layer MLP with $64$ nodes each. We project the data $x$ through a two-layer MLP with $128$ nodes each. Both embedding networks use SiLU activations. We then concatenate the embeddings of time $t$ and data $x$ before projecting the resulting vector and the parameter values $\theta$ through the transformations $f = (f^{(1)},\dots, f^{(K)})$. In our experiments, we set $K=4$ and use a block size of $64$, i.e.,
\begin{align*}
d_\mathrm{in}^{(1)} &= 1, & d_\mathrm{out}^{(1)} &= 64,\\
d_\mathrm{in}^{(2)} &= 64, & d_\mathrm{out}^{(2)} &= 64,\\
d_\mathrm{in}^{(3)} &= 64, & d_\mathrm{out}^{(3)} &= 64,\\
d_\mathrm{in}^{(4)} &= 64, & d_\mathrm{out}^{(4)} &= 1.
\end{align*}
Each $f^{(k)}$ is designed as described in Equation~\eqref{eqn:linear-projection}, where we use the identity function for $g$ and apply $\mathrm{act} := \tanh$ after every layer except the last. After each intermediate layer $f^{(2)}, \dots, f^{(K-1)}$, we additionally reverse the order of the hidden units belonging to each parameter, such that all dimensions of a multi-dimensional parameter can depend on each other.We condition on data and time by adding the concatenation of the data and time embeddings after the first projection $f^{(1)}$, i.e.,
\begin{align*}
    \theta' &= f^{(1)}(\theta) \\
    \theta'' &= \theta' + \texttt{Linear}(c)
\end{align*}
where we denote with \texttt{Linear} a linear neural network layer (with bias) and with $c$ the concatenation of time and data embeddings. We denote with $\lambda_t$ the composition of the embeddings and the projection through $f$. Its output is combined with the input through the convex combination in Equation~\eqref{eqn:convex-combination}.

\paragraph{BNPE} BNPE uses the same architecture as CPE, but without causal masks, such that the weight matrices of all transformations $f^{(1)}, \dots, f^{(K)}$ are dense.

\paragraph{FMPE} FMPE uses a two-layer MLP with 128 nodes each to embed parameter values $\theta$ and a two-layer MLP with 128 nodes each to embed data values $x$. We first project time values $t$ through a sinusoidal embedding of dimension $64$, before embedding the result using a two-layer MLP with 64 nodes each. We then concatenate the embedded parameter, data and time values and project them through an MLP with two hidden layers of 256 nodes each. FMPE uses the forward process and vector field as described in \citet{lipman2023flow}, and we set $\sigma_\mathrm{min} = 0.001$. All activation functions are GELUs \citep{hendrycks2016gaussian}.

\paragraph{AIO} Following the original publication \citep{gloeckler2024allinone}, AIO uses a transformer backbone as a score model with a token dimension of $50$, an attention size of $10$, and a widening factor of $3$. It uses six transformer layers and four attention heads for all experiments. We use embedding dimensions of 32, 10 and 8 for values, ids and conditioning variables, respectively, and embed time using a sinusoidal embedding followed by a two-layer MLP with 64 nodes each. We use a dropout rate of $0.0$. All activation functions are GELUs \citep{hendrycks2016gaussian}. AIO uses the variance-preserving SDE of \citet{song2021scorebased} with $\beta_\mathrm{min} = 0.1$ and $\beta_\mathrm{max} = 10$. For AIO, we use the undirected graph of the model as attention mask. A directed graph would require recomputing the attention mask whenever the conditioning changes, e.g., at every training batch when the condition mask is sampled randomly \citep[Section 3.2]{gloeckler2024allinone}, which makes it impractical to use with AIO. For AIO (no mask), we use no attention mask.

For the SIR model, we use AIO without an attention mask: since every summary statistic is computed from the same simulated time series, which depends on both parameters, the graph of the model is fully connected and the mask would not remove any dependencies.

\paragraph{PSE} To be able to compare PSE fairly to FMPE and AIO, we use the same neural network architecture as for FMPE (since the mathematical frameworks are related) and the same variance-preserving SDE as for AIO with $\beta_\mathrm{min} = 0.1$ and $\beta_\mathrm{max} = 10$. All activation functions are GELUs \citep{hendrycks2016gaussian}.

\subsection{Training}
All methods, i.e., CPE, BNPE, and all baselines, are trained in exactly the same way. Given a simulation budget of $N$ pairs $(x_n, \theta_n)$, we hold out $10\%$ of the pairs as a validation set and train on the remaining $90\%$. We use the Adam optimizer \citep{kingma2015adam} with learning rate $10^{-4}$, $\beta_1 = 0.9$, $\beta_2 = 0.999$ and $\epsilon = 10^{-8}$, and a batch size of $200$. After each epoch, i.e., one pass over the training set, we compute the loss on the validation set. An epoch counts as an improvement if its validation loss is lower than that of the last improving epoch by more than $\delta = 0.01$. Training stops as soon as more than $20$ consecutive epochs pass without improvement, or after a maximum of $2\,000$ epochs, whichever occurs first. We then use the parameters that achieved the lowest validation loss during training. We use \texttt{Optax} for gradient-based optimization \citep{deepmind2020jax}.

\subsection{Sampling}
For each method, we draw $20\,000$ posterior samples. CPE-RK and all baselines integrate the ODE with the adaptive Runge--Kutta 5(4) solver of \texttt{SciPy} \citep{virtanen2022scipy} with relative and absolute tolerances of $10^{-5}$. For the baselines, we use the implementations of \texttt{sbijax} \citep{dirmeier2024simulation}. CPE integrates the ODE from $t = 10^{-3}$ to $t = 1$. For the CPE variant CPE-Euler, we use a custom Euler solver with $T=20$ steps and step size $\mathrm{d}t = 1/20$.

\subsection{Computational resources}
All MCMC sampling, model training and model evaluations were run on AMD EPYC\texttrademark{} 7742 processors with $64$ cores and 256\,GB of RAM. Runtimes of all methods were between $2$ and $6$\,h. In total, we conducted $950$ experiments ($5$ seeds $\times$ $10$ experimental models $\times$ ($6$ methods $\times$ $3$ simulation budgets $+$ $1$ MCMC reference)), yielding a total runtime of approximately $950 \times 4 \times 64 = 243\,200$ CPU hours. We use the workflow manager Snakemake \citep{koster2012snakemake} to run all experiments automatically on a Slurm cluster.

\section{Additional results}
\label{app:additional-results}

\subsection{Benchmark tasks}

In this section, we report additional results for the benchmark models. Figure~\ref{app:benchmark_tasks-c2st-10000} shows the C2ST of all methods for a simulation budget of $N = 10\,000$, complementing the H-min results in Figure~\ref{fig:benchmark_tasks-hmin} of the main text. Figure~\ref{app:gaussian_linear-posteriors} shows the posterior approximations of all methods for the linear Gaussian model, on which the two metrics disagree.

\paragraph{Gaussian} The H-min and C2ST results are generally consistent across benchmark models, except for the linear Gaussian model (Gaussian in Figure~\ref{fig:benchmark_tasks-hmin} of the main text), where AIO performs considerably worse under H-min than under C2ST (Figure~\ref{app:benchmark_tasks-c2st-10000}). Like \citet{dirmeier2023simulation}, we find that the H-min divergence, although harder to compare across benchmark models, is the more reliable metric in this case, since it is more sensitive to differences in the geometry of the approximate and reference posteriors. Indeed, the posterior plots in Figure~\ref{app:gaussian_linear-posteriors} show that the AIO posterior is degenerate even though it achieves low C2ST values.

\begin{figure}[H]
    \centering
    \includegraphics[width=\textwidth]{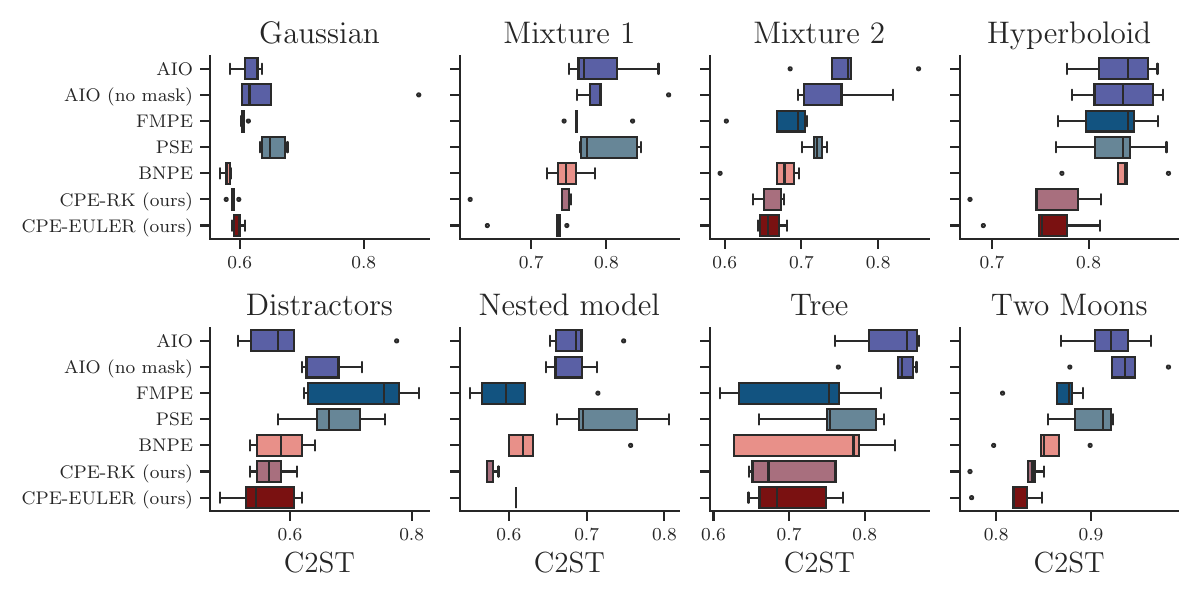}
    \caption{C2ST of CPE and the baselines on the eight benchmark models for a simulation budget of $N = 10\,000$. Values closer to $0.5$ are better. CPE-RK and CPE-Euler denote CPE with an adaptive Runge--Kutta 5(4) solver and a $20$-step Euler solver, respectively.}
    \label{app:benchmark_tasks-c2st-10000}
\end{figure}

\begin{figure}[H]
    \centering
    \begin{subfigure}[b]{\textwidth}
        \centering
        \includegraphics[width=\textwidth]{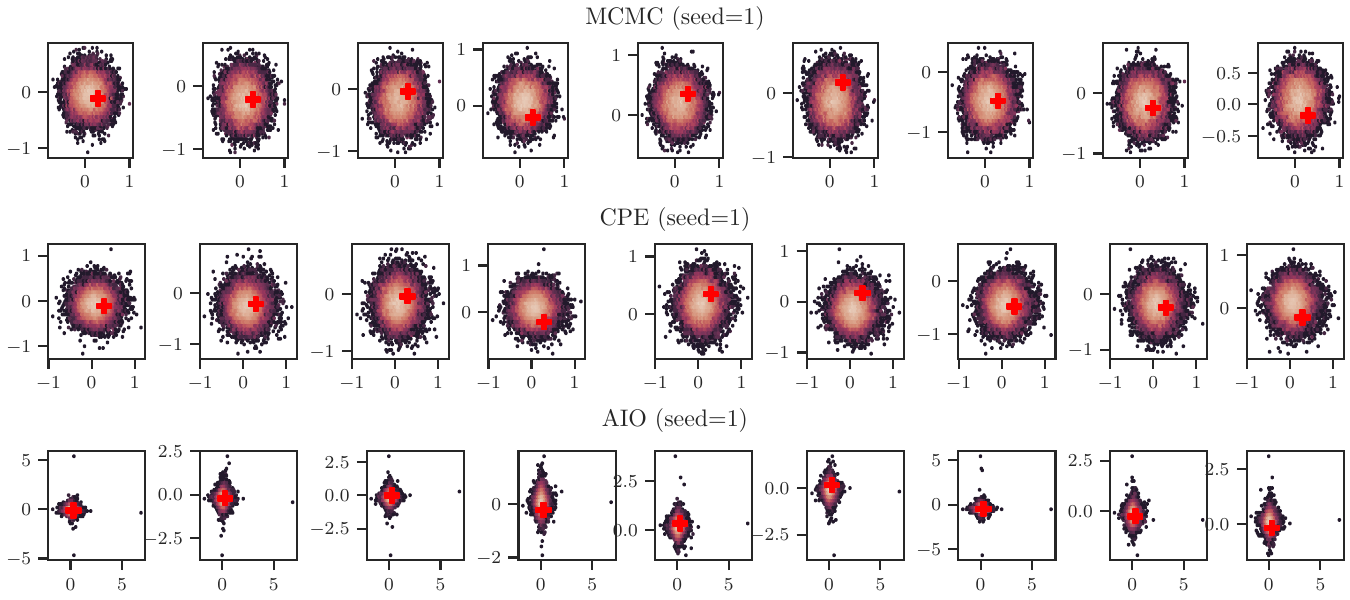}      
        \caption{Seed 1}
    \end{subfigure}    
    \vspace{1.5em} 
    
    \begin{subfigure}[b]{\textwidth}
        \centering
        \includegraphics[width=\textwidth]{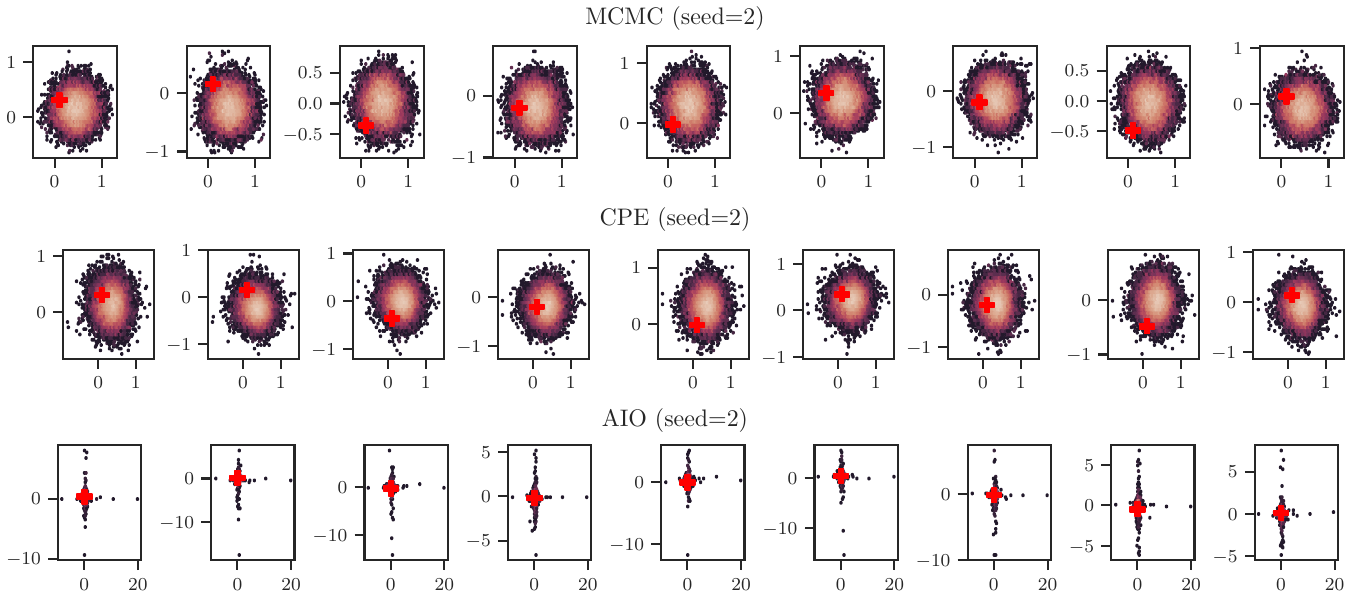}    
        \caption{Seed 2}
    \end{subfigure}
    \caption{Posterior distribtuions of CPE and AIO for the linear Gaussian model. The AIO posterior is degenerate for all five seeds (only two out of five seeds shown), although it achieves low C2ST values (Figure~\ref{app:benchmark_tasks-c2st-10000}).}
    \label{app:gaussian_linear-posteriors}
\end{figure}

\subsection{SIR model}
\label{app:additional_results_sir}

We report the H-min and C2ST of all methods on the SIR model for all three simulation budgets (Figure~\ref{app:sir-evaluations}), complementing the H-min results in the main text. In addition, we show the posterior approximations of all methods (Figure~\ref{app:sir-posteriors}) and their posterior predictive distributions, obtained by simulating data from posterior samples (Figure~\ref{app:sir-ppds}). Since the summary statistics of the SIR model are informative but not sufficient for the parameters, the posterior predictive distributions indicate how closely each method recovers the observed data.

\begin{figure}[H]
    \centering
     \begin{subfigure}[b]{0.3\textwidth}
         \centering
         \includegraphics[width=\textwidth]{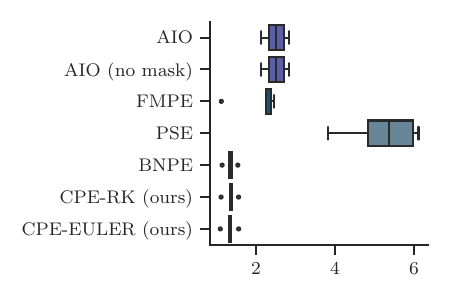}
         \caption{H-min, $N = 1\,000$}
     \end{subfigure}
     \hfill
     \begin{subfigure}[b]{0.3\textwidth}
         \centering
         \includegraphics[width=\textwidth]{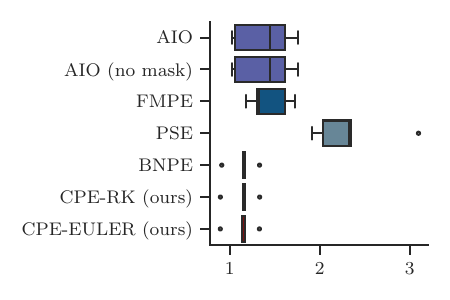}
         \caption{H-min, $N = 10\,000$}
     \end{subfigure}
     \hfill
     \begin{subfigure}[b]{0.3\textwidth}
         \centering
         \includegraphics[width=\textwidth]{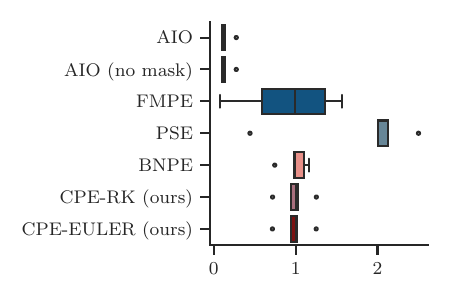}
         \caption{H-min, $N = 100\,000$}    
     \end{subfigure}
        
     \begin{subfigure}[b]{0.3\textwidth}
         \centering
         \includegraphics[width=\textwidth]{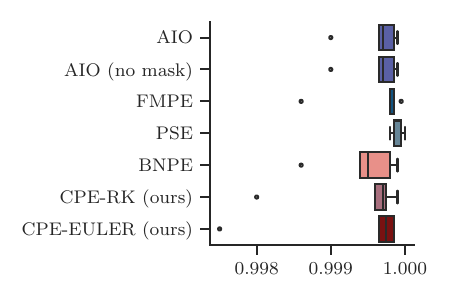}
         \caption{C2ST, $N = 1\,000$}
     \end{subfigure}
     \hfill
     \begin{subfigure}[b]{0.3\textwidth}
         \centering
         \includegraphics[width=\textwidth]{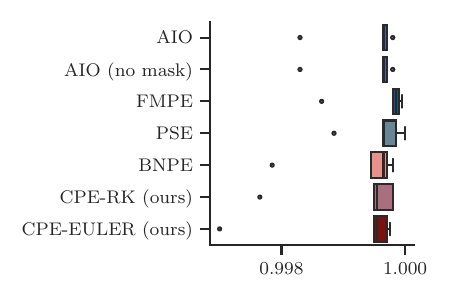}
         \caption{C2ST, $N = 10\,000$}    
     \end{subfigure}
     \hfill
     \begin{subfigure}[b]{0.3\textwidth}
         \centering
         \includegraphics[width=\textwidth]{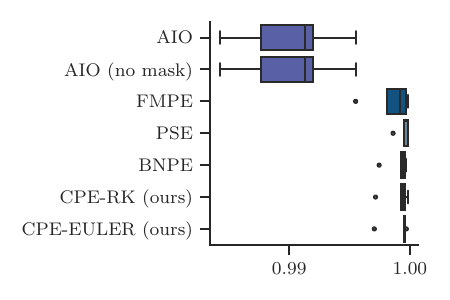}
         \caption{C2ST, $N = 100\,000$}
     \end{subfigure}
    \caption{H-min (top row) and C2ST (bottom row) of all methods on the SIR model for simulation budgets $N \in \{1\,000, 10\,000, 100\,000\}$. Lower H-min values and C2ST values closer to $0.5$ are better.}
    \label{app:sir-evaluations}
\end{figure}

\begin{figure}[H]
\centering
\includegraphics[width=1\textwidth]{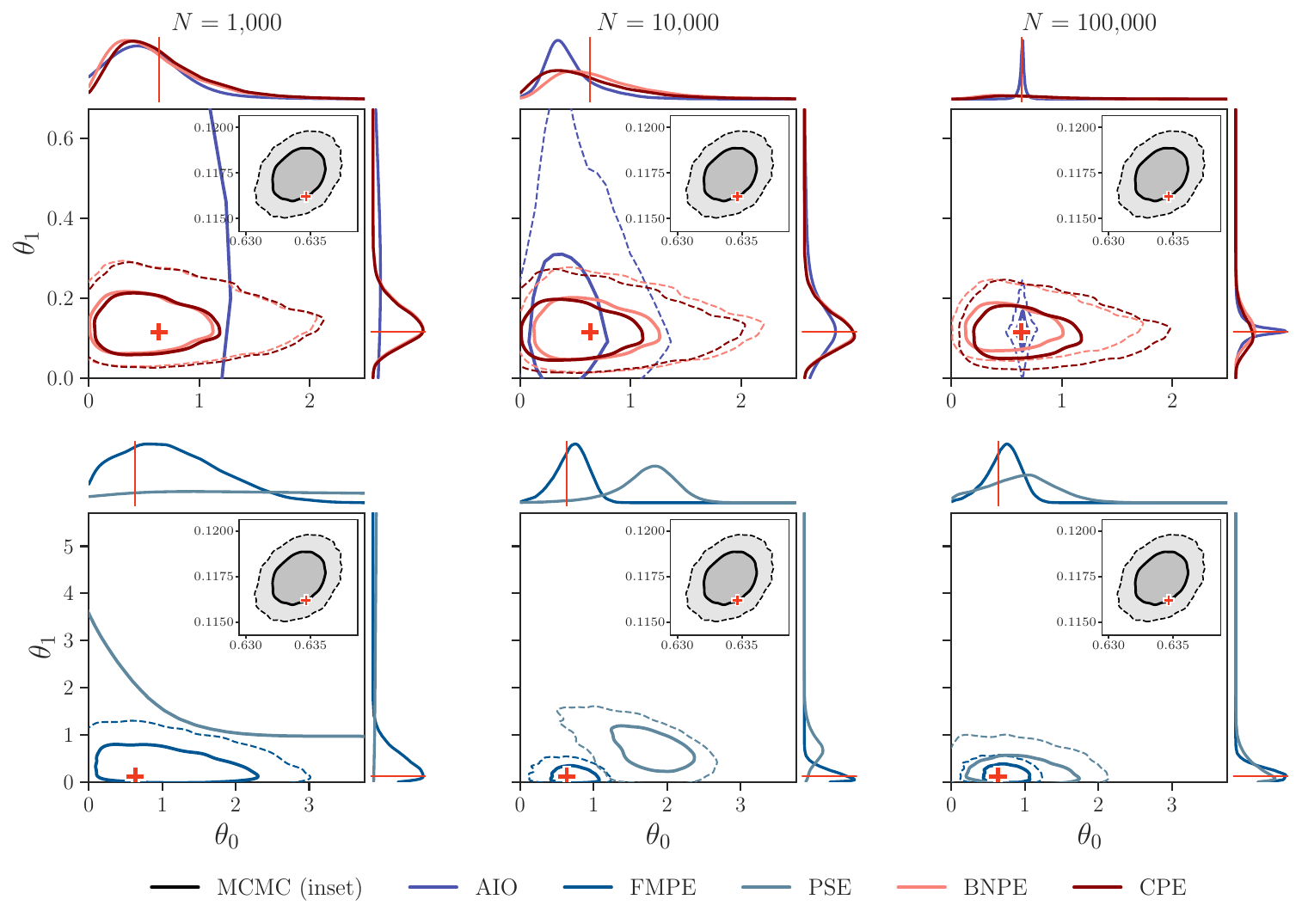}
\caption{Posterior distributions of all methods for the SIR model, where $\theta_0$ denotes the contact rate and $\theta_1$ the recovery rate.}
\label{app:sir-posteriors}    
\end{figure}

\begin{figure}[H]
\centering
\includegraphics[width=1\textwidth]{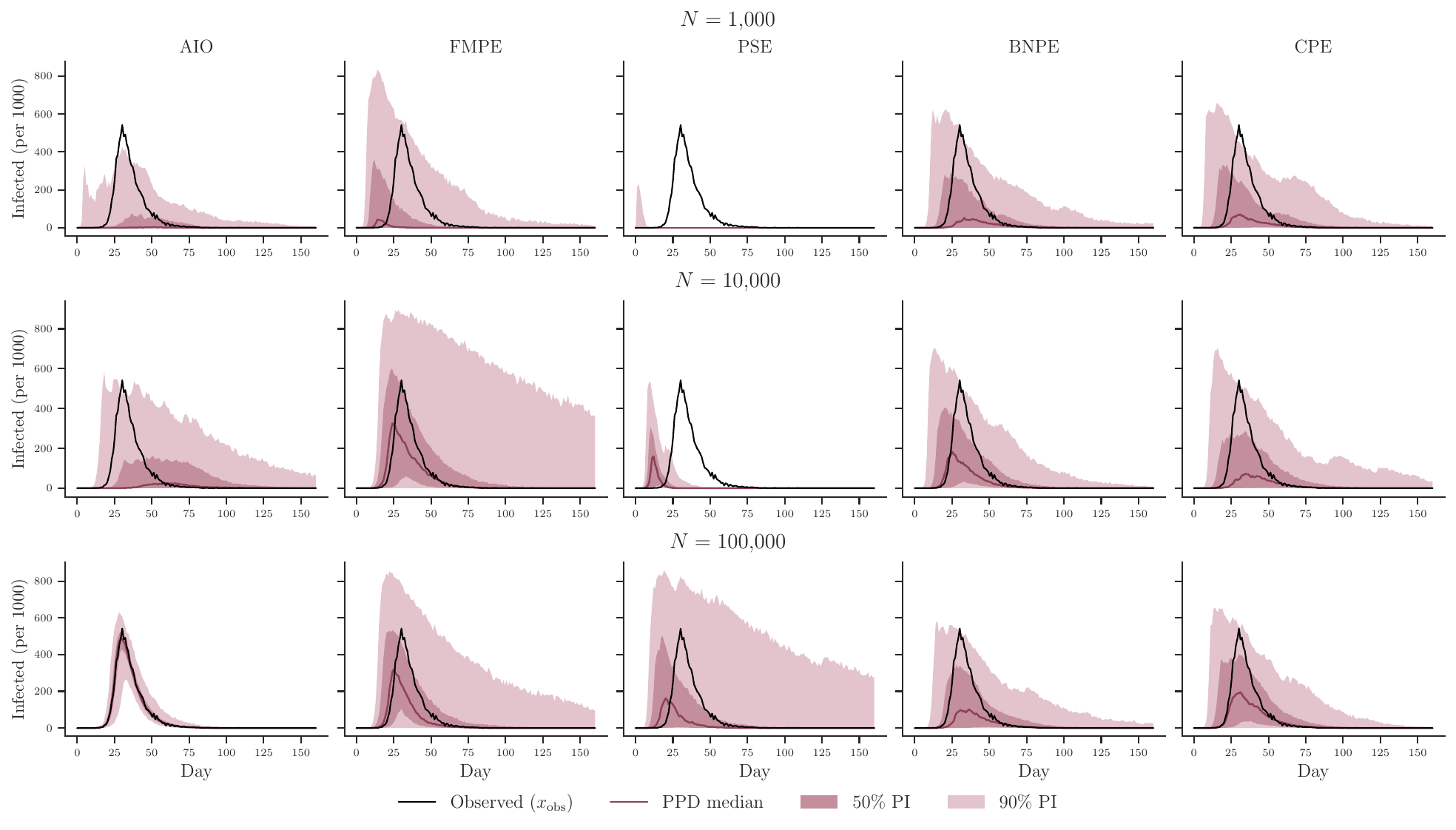}
\caption{Posterior predictive distributions of all methods for the SIR model at a simulation budget of $N = 100\,000$. The posterior predictive distributions of AIO closely reproduce the observed data, whereas those of the other methods have considerably higher variance, as expected for summary statistics that are not sufficient for the parameters.}
\label{app:sir-ppds}    
\end{figure}

\subsection{Hierarchical model}
\label{app:additional_results_hierarchical_model}

We report the H-min and C2ST of all methods on the hierarchical model for all three simulation budgets (Figure~\ref{app:hierarchical_model-evaluations}), complementing the H-min results in the main text. In addition, we show the posterior approximations of all methods for each simulation budget (Figure~\ref{app:hierarchical-posteriors}) and their posterior predictive distributions, obtained by simulating data from posterior samples (Figure~\ref{app:hierarchical-ppds}).

\begin{figure}[H]
    \centering
     \begin{subfigure}[b]{0.3\textwidth}
         \centering
         \includegraphics[width=\textwidth]{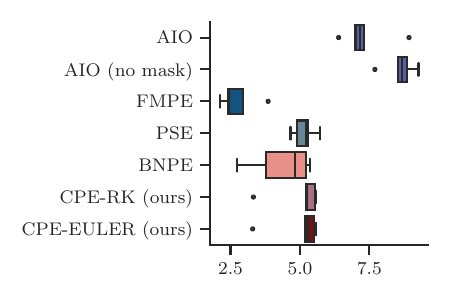}
         \caption{H-min, $N = 1\,000$}
     \end{subfigure}
     \hfill
     \begin{subfigure}[b]{0.3\textwidth}
         \centering
         \includegraphics[width=\textwidth]{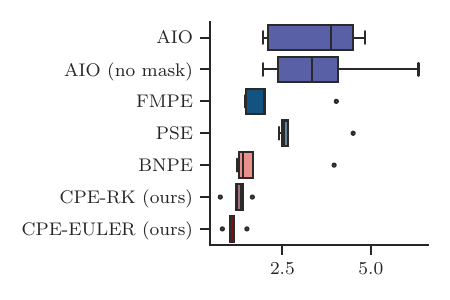}
         \caption{H-min, $N = 10\,000$}
     \end{subfigure}
     \hfill
     \begin{subfigure}[b]{0.3\textwidth}
         \centering
         \includegraphics[width=\textwidth]{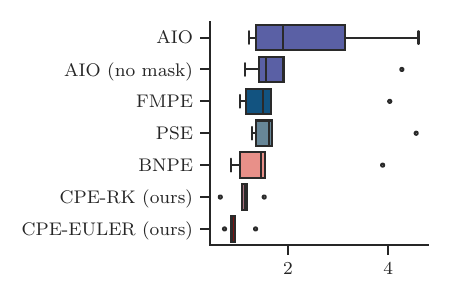}
         \caption{H-min, $N = 100\,000$}    
     \end{subfigure}
        
     \begin{subfigure}[b]{0.3\textwidth}
         \centering
         \includegraphics[width=\textwidth]{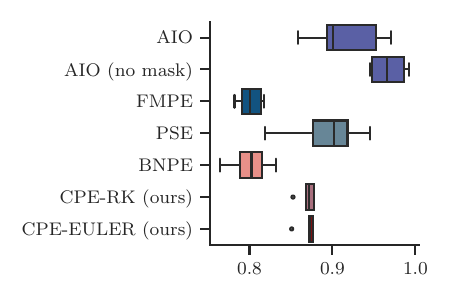}
         \caption{C2ST, $N = 1\,000$}
     \end{subfigure}
     \hfill
     \begin{subfigure}[b]{0.3\textwidth}
         \centering
         \includegraphics[width=\textwidth]{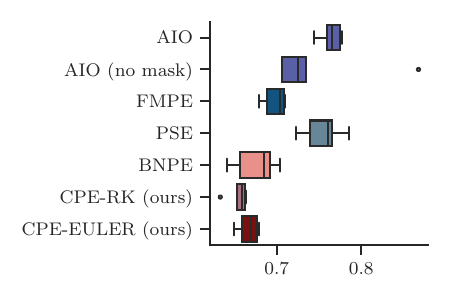}
         \caption{C2ST, $N = 10\,000$}    
     \end{subfigure}
     \hfill
     \begin{subfigure}[b]{0.3\textwidth}
         \centering
         \includegraphics[width=\textwidth]{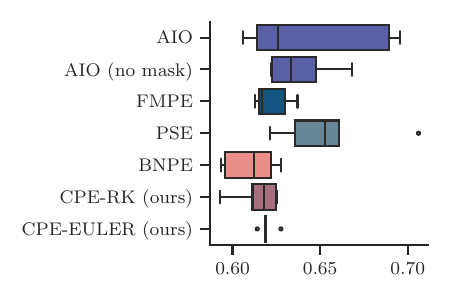}
         \caption{C2ST, $N = 100\,000$}
     \end{subfigure}
    \caption{H-min (top row) and C2ST (bottom row) of all methods on the hierarchical model for simulation budgets $N \in \{1\,000, 10\,000, 100\,000\}$. Lower H-min values and C2ST values closer to $0.5$ are better.}
    \label{app:hierarchical_model-evaluations}
\end{figure}

\begin{figure}[H]
\centering
\begin{subfigure}{\textwidth}
  \centering
  \includegraphics[width=\textwidth]{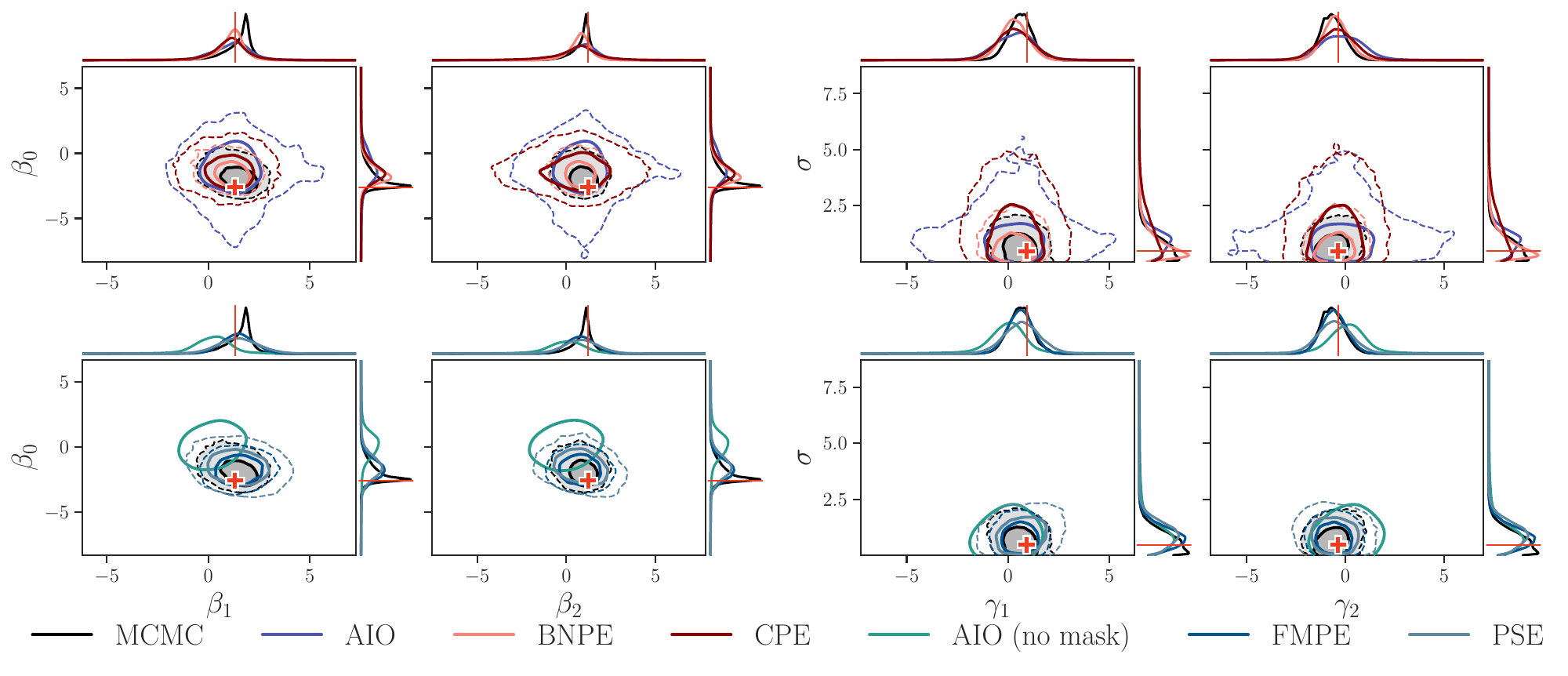}
  \caption{$N = 1\,000$}
  \label{app:hierarchical-posteriors-1000}
\end{subfigure}
\begin{subfigure}{\textwidth}
  \centering
  \includegraphics[width=\textwidth]{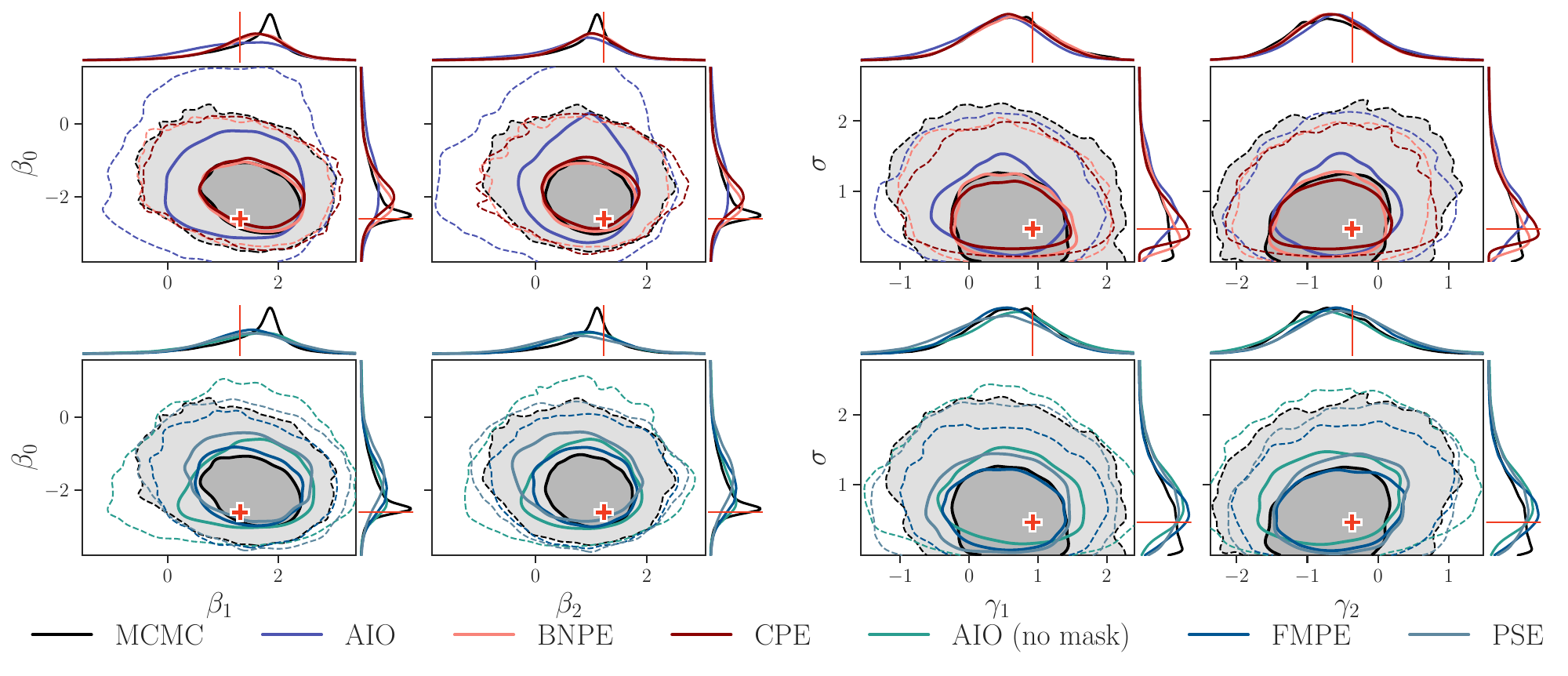}
  \caption{$N = 10\,000$}
  \label{app:hierarchical-posteriors-10000}
\end{subfigure}
\begin{subfigure}{\textwidth}
  \centering
  \includegraphics[width=\textwidth]{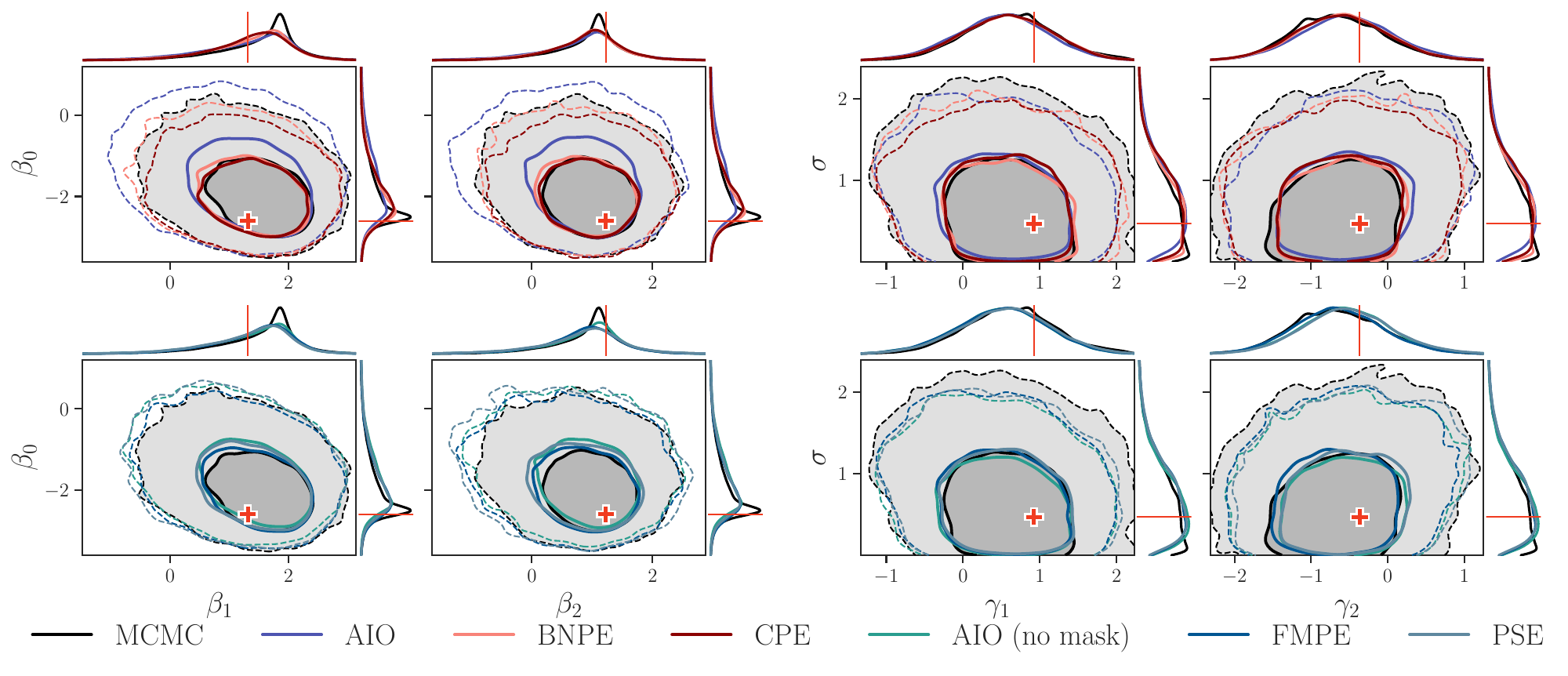}
  \caption{$N = 100\,000$}
  \label{app:hierarchical-posteriors-100000}
\end{subfigure}
\caption{Posterior distributions of all methods for the hierarchical model for simulation budgets $N \in \{1\,000, 10\,000, 100\,000\}$, showing the lower-level locations $\beta$, the top-level location $\gamma$ and the observation scale $\sigma$.}
\label{app:hierarchical-posteriors}
\end{figure}

\begin{figure}[H]
\centering
\includegraphics[width=\textwidth]{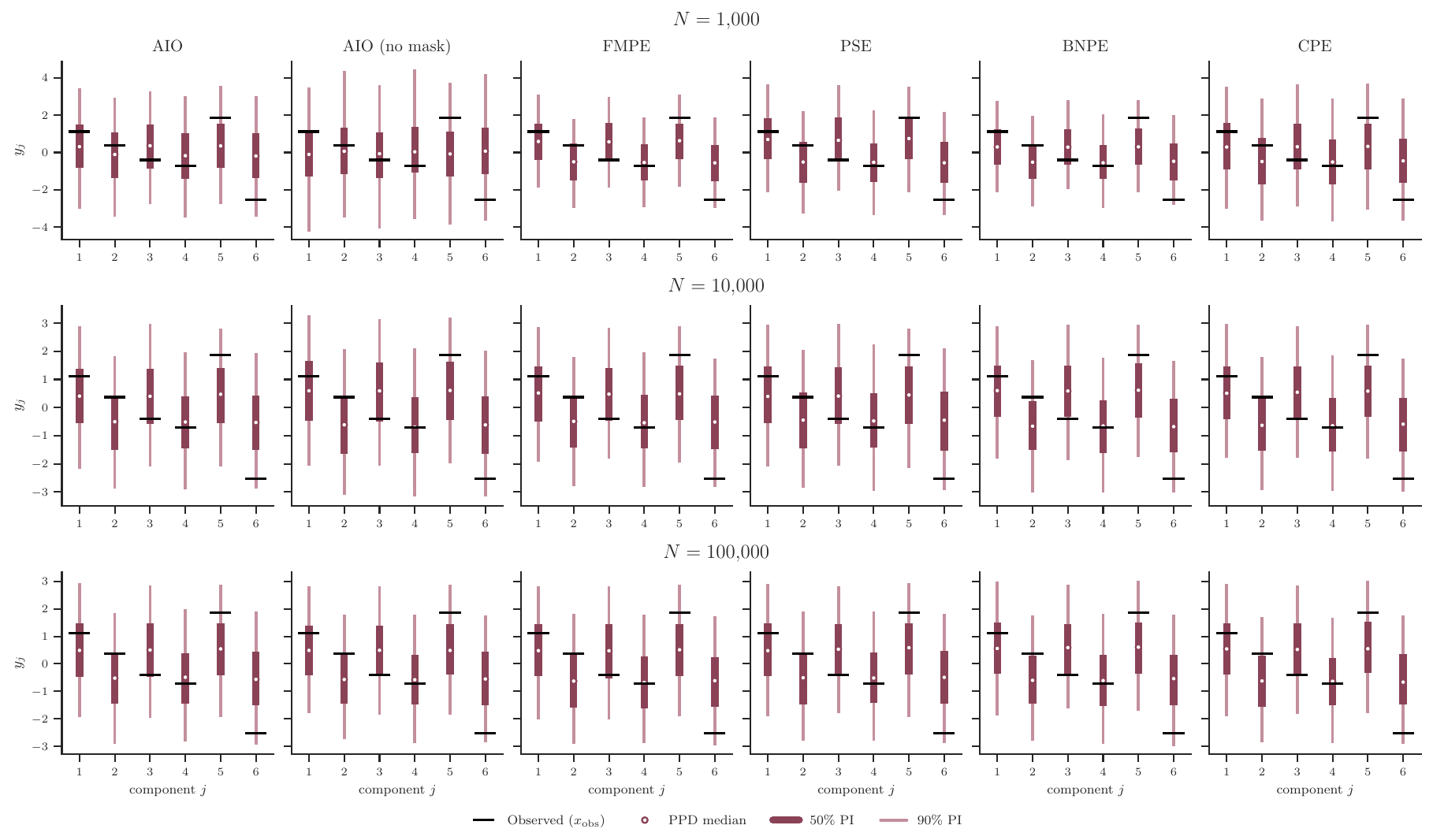}
\caption{Posterior predictive distributions of all methods for the hierarchical model.}
\label{app:hierarchical-ppds}    
\end{figure}

\end{document}